\documentclass{article}

\usepackage[numbers,sort]{natbib}

    \usepackage[preprint]{neurips_2020}

\usepackage[utf8]{inputenc} 
\usepackage[T1]{fontenc}    
\usepackage{hyperref}       
\usepackage{url}            
\usepackage{booktabs}       
\usepackage{amsfonts}       
\usepackage{nicefrac}       
\usepackage{microtype}      
\usepackage{stmaryrd}

\newcommand{\st}{\text{ subject to }}

\newcommand{\prob}[1]{\Pr\left(\left\{#1\right\}\right)}
\newcommand{\proband}[2]{\Pr\left(\left\{#1\right\}\cap\left\{#2\right\}\right)}
\newcommand{\probor}[2]{\Pr\left(\left\{#1\right\}\cup\left\{#2\right\}\right)}

\newcommand{\probcond}[2]{\Pr\left(\left\{#1\right\}\Big|\left\{#2\right\}\right)}

\newcommand{\probandcond}[3]{\Pr\left(#3\left\{#1\right\}\Big|\left\{#2\right\}\right)}
\renewcommand{\forall}{\text{for all }}
\newcommand{\defeq}{\vcentcolon=}

\newcommand{\Xtr}{{\mathbf{X}_{\text{train}}}}
\newcommand{\Ytr}{\mathbf{Y}_{\text{train}}}
\newcommand{\etr}{\epsilon_{\text{train}}}
\newcommand{\X}{\mathbf{X}}

\newcommand{\ea}{\epsilon}
\newcommand{\betaBP}{\hat{\beta}^{\text{BP}}}
\newcommand{\wBP}{w^{\text{BP}}}

\newcommand{\ith}[1]{#1[i]}
\newcommand{\hi}{h_i}
\newcommand{\mycap}{\mathop{\cap}}

\newenvironment{myproof}[1] {\textbf{Proof of {#1}: }}{\hfill$\blacksquare$}

\usepackage{hyperref}
\usepackage{xcolor}
\hypersetup{
    colorlinks,
    linkcolor={red!50!black},
    citecolor={blue!50!black},
    urlcolor={blue!80!black}
}

\newcommand{\underbeta}{\underline{\beta}}
\newcommand{\hatunderbeta}{\hat{\underline{\beta}}}
\newcommand{\underw}{\underline{w}}

\newcommand{\underbetaBP}{\hat{\underline{\beta}}^{\text{BP}}}
\newcommand{\underwBP}{\underline{w}^{\text{BP}}}
\newcommand{\wltwo}{\underline{w}^{\ell_2}}
\newcommand{\ex}{\mathds{E}}
\newcommand{\SNR}{\mathrm{SNR}}

\renewcommand{\citep}{\cite}
\usepackage{mathrsfs}
\usepackage{comment}
\usepackage{amsmath, bm}
\usepackage{dsfont}
\usepackage{mathtools}
\usepackage{wrapfig}
\usepackage{amsthm}
\usepackage{amssymb}

\newtheorem{theorem}{Theorem}
\newtheorem{corollary}[theorem]{Corollary}
\newtheorem{lemma}[theorem]{Lemma}

\newtheorem{proposition}[theorem]{Proposition}

\DeclareMathOperator*{\argmax}{arg\,max}
\DeclareMathOperator*{\argmin}{arg\,min}

\newcommand{\abs}[1]{|#1|}

\title{Overfitting Can Be Harmless for Basis Pursuit,\\
But Only to a Degree}

\author{%
  Peizhong Ju\\
  School of ECE\\
  Purdue University\\
  West Lafayette, IN 47906\\
  \texttt{jup@purdue.edu}\\
  \And
  Xiaojun Lin\\
  School of ECE\\
  Purdue University\\
  West Lafayette, IN 47906\\
  \texttt{linx@purdue.edu}\\
  \And
  Jia Liu\\
  Department of ECE\\
  The Ohio State University\\
  Columbus, OH 43210\\
  \texttt{liu@ece.osu.edu}
}

\begin{document}

\maketitle

\begin{abstract}


Recently, there have been significant interests in studying
the so-called ``double-descent'' of the 
	generalization error of linear regression models under the
	overparameterized and overfitting regime, with the hope that
	such analysis may provide the first step towards understanding
	why overparameterized
	deep neural networks (DNN) still generalize well.
	However, to date most of these
	studies focused on the min $\ell_2$-norm solution that overfits
	the data. 
	In contrast, in this paper we study the overfitting
	solution that minimizes the $\ell_1$-norm, which is known as
	Basis Pursuit (BP) in the compressed sensing literature. Under a
	sparse true linear regression model with $p$ \emph{i.i.d.}
	Gaussian features, we show that for a large range of $p$ up to a
	limit that grows
	exponentially with the number of samples $n$, with high
	probability the model error of BP is upper bounded by a value
	that decreases with $p$.
	To the best of our
	knowledge, this is the first analytical result in the literature 
	establishing the double-descent of overfitting BP for finite $n$
	and $p$. Further, our results reveal significant differences
	between the double-descent of BP and min $\ell_2$-norm solutions. 
	Specifically, the double-descent upper-bound of BP is independent of the
	signal strength, and for high $\SNR$ and sparse models the
	descent-floor of BP can be much lower and wider than that of min
	$\ell_2$-norm solutions. 

\end{abstract}

\section{Introduction}

One of the mysteries of deep neural networks (DNN) is that they are not
only heavily over-parameterized so that they can fit the training data (even
perturbed with noise) nearly perfectly \cite{cybenko1989approximation,telgarsky2016benefits,barron1993universal,barron1994approximation}, but still produce models that
generalize well to new test data \cite{zhang2016understanding,
advani2017high}. This combination of overfitting and good
generalization challenges the classical wisdom in statistical learning
(e.g., the well-known bias-variance tradeoff) 
\cite{bishop2006pattern, hastie2009elements, stein1956inadmissibility, james1992estimation,lecun1991second,tikhonov1943stability}. 
%
%
As a first step towards
understanding why overparameterization and overfitting may be
harmless, a line of recent work has focused on linear regression models 
\cite{belkin2018understand, belkin2019two,
bartlett2019benign,
hastie2019surprises, mei2019generalization,muthukumar2019harmless}. Indeed, such
results have demonstrated an interesting ``double descent'' phenomenon
for linear models. Roughly
speaking, let $n$ be the number of training samples, and $p$ be the
number of parameters of a linear regression model. As $p$ approaches $n$
from below, the test error of the model (that tries to best fit the
training data) first decreases and then increases to infinity, which is
consistent with the well-understood bias-variance trade-off. As $p$
further increases beyond $n$, overfitting starts to occur (i.e., the
training error will always be zero). However, if one chooses the
overfitting solution that minimizes the $\ell_2$-norm, the test error
decreases again as $p$ increases further. 
Such observations, although for models very different from DNN, provide
some hint why overfitting solutions may still generalize well. 

To date most studies along this direction have focused on the minimum
$\ell_2$-norm overfitting solutions \cite{belkin2019two,
muthukumar2019harmless,hastie2019surprises,mei2019generalization,mitra2019understanding}. One possible motivation is that, at
least for linear regression problems, 
Stochastic Gradient Descent (SGD), which is often used to
train DNNs, is believed to produce the min $\ell_2$\nobreakdash-\hspace{0pt}norm overfitting
solutions
\cite{zhang2016understanding}.
%
%
However, it is unclear whether SGD will still produce the min $\ell_2$-norm
overfitting solutions for more general models, such as DNN. 
Therefore, it is important to understand whether and how double-descent
occurs for other types of overfitting solutions. Further, the min
$\ell_2$-norm solution usually does not promote sparsity. Instead, it
tends to yield small weights spread across nearly all features, which
leads to distinct characteristics of its double-descent curve (see
further comparisons below). The double-descent of other overfitting
solutions will likely have different characteristics. By understanding
these differences, we may be able to discern which type of overfitting
solutions may approximate the generalization power of DNN better.

In this paper, 
we focus on the overfitting
solution with the minimum $\ell_1$-norm. This is known as Basis Pursuit
(BP) in the compressed sensing literature \cite{chen2001atomic, chen1995examples}. 
There are several reasons why we are interested in BP with overfitting.
First, similar to $\ell_2$-minimization, it does not involve any explicit
regularization parameters, and thus can be used even if we do not know
the sparsity level or the
noise level. Second, it is well-known that using $\ell_1$-norm
promotes sparse solutions \cite{donoho2005stable,zhao2006model, meinshausen2009lasso, bickel2009simultaneous, tibshirani1996regression,chen1995examples}, which is useful in the
overparameterized regime. 
Third, it is known that the $\ell_1$-norm of the model is
closely related to its ``fat-shattering dimension,'' which is also
related to the Vapnik-Chervonenkis (V-C) dimension and the model capacity 
\cite{bartlett1998sample}. Thus, BP seems to have the
appealing flavor of ``Occam's razor'' \cite{blumer1987occam}, i.e., to
use the simplest explanation that matches the training data.
However, until now the double descent of BP has not been well studied. 
The numerical results in
\cite{muthukumar2019harmless} suggest that, for a wide range of $p$, BP
indeed exhibits double-descent and produces low test-errors.
However, no analysis is provided
in \cite{muthukumar2019harmless}. In the compressed sensing literature,
test-error bounds for BP were provided for the
overparameterized regime, see, e.g., \cite{donoho2005stable}.
However, the notion of BP therein is different as it requires that the
model does
\emph{not} overfit the training data. Hence, such results 
cannot be used to explain the 
``double-descent'' of BP in the overfitting regime. For classification problems, minimum $\ell_1$-norm solutions that interpolate categorical data were studied in \cite{liang2020precise}. However, the notion of ``overfitting'' in \cite{liang2020precise} for classification problems is quite different from that for the regression problems studied in this paper and \cite{muthukumar2019harmless}.
To the best of our knowledge, the only work that analyzed the
double-descent of overfitting BP in a regression setting is \cite{mitra2019understanding}. However, 
this work mainly 
studies the setting where both $n$ and $p$ grow to infinity at a fixed
ratio. As we show later, BP exhibits interesting dynamics of
double-descent when $p$ is exponentially larger than $n$, which
unfortunately collapsed into
a single point of $p/n \to \infty$ in the setting of \cite{mitra2019understanding}. 
%
In summary, a thorough study of the double-descent of BP for finite $n$
and $p$ is still missing. 

The main contribution of this paper is thus to provide new analytical
bounds on the model error of BP in the overparameterized regime. As in
\cite{belkin2019two}, we
consider a simple linear regression model with $p$ \emph{i.i.d.}
Gaussian features. We assume that the true model is sparse, and the
sparsity level is $s$. BP is used to train the model by exactly fitting
$n$ training samples. For a range of $p$ up to a value that grows
exponentially with
$n$, we show an upper bound on the model error that decreases with $p$,
which explains the ``double descent'' phenomenon observed for BP in the
numerical results in \cite{muthukumar2019harmless}. 
To the best of our knowledge, this is the first analytical result in the 
literature establishing the double-descent of min $\ell_1$-norm
overfitting solutions for finite $n$ and $p$. 

Our results reveal significant differences between the double-descent of
BP and min $\ell_2$-norm solutions. First, our upper bound for
the model error of BP
is independent of the signal strength (i.e., $\|\beta\|_2$ in the model in
Eq.~(\ref{eq.data})). In contrast, the double descent of min $\ell_2$-norm
overfitting solutions usually increases with $\|\beta\|_2$, suggesting
that some signals are ``spilled'' into the model error.
Second, the double descent of BP is much slower (polynomial in $\log
p$) than that of min $\ell_2$-norm overfitting solutions (polynomial in
$p$). On the other hand, this also means that the double-descent of BP
manifests over a
larger range of $p$ and is easier to observe than that of min
$\ell_2$-norm overfitting solutions. Third, with both $\ell_1$-norm and
$\ell_2$-norm minimization, there is a ``descent floor'' where the model
error reaches the lowest level. However, at high signal-to-noise ratio
($\SNR$) and for sparse models, the descent floor of BP is both lower (by a factor proportional to
$\sqrt{\frac{1}{s}\sqrt{\SNR}}$) and wider (for a range of $p$ exponential
in $n$) than that of min $\ell_2$-norm solutions. 

One additional insight revealed by our proof is the 
connection between the model error of BP and the ability for an
overparameterized model to 
fit \emph{only} the noise. Roughly speaking, as long as the model is able to 
fit only the noise with small $\ell_1$-norm solutions, BP will also
produce small model errors (see Section~\ref{sect:proof}). This
behavior also appears to be unique to BP. 

Finally, our results also reveal certain limitations of
overfitting, i.e., the descent floor of either BP or min $\ell_2$-norm
solutions cannot be as low as regularized (and non-overfitting)
solutions such as LASSO~\cite{tibshirani1996regression}. However, the type of explicit
regularization used in LASSO is usually not used in DNNs. Thus, it
remains an open question how to find practical overfitting solutions
that can achieve even lower generalization errors without any  explicit regularization.

\section{Problem setting}\label{sec.system_model}
Consider a linear model as follows: 
\begin{align}\label{eq.model}
    y=x^T{\underbeta}+\ea,
\end{align}
where $x\in\mathds{R}^{p}$ is a vector of $p$ features, $y\in\mathds{R}$ denotes the output, $\ea\in \mathds{R}$ denotes the noise, and ${\underbeta}\in\mathds{R}^p$ denotes the regressor vector. We assume that each element of $x$ follows \emph{i.i.d.} standard Gaussian distribution, and $\ea$ follows independent Gaussian distribution with zero mean and variance $\sigma^2$. Let $s$ denote the sparsity of ${\underbeta}$, i.e., ${\underbeta}$ has at most $s$ non-zero elements. Without loss of generality, we assume that all non-zero elements of ${\underbeta}$ are in the first $s$ elements. For any $p\times 1$ vector $\alpha$ (such as ${\underbeta}$), we use $\ith{\alpha}$ to denote its $i$-th element, use $\alpha_0$ to denote the $s\times 1$ vector that consists of the first $s$ elements of $\alpha$, and use $\alpha_1$ to denote the $(p-s)\times 1$ vector that consists of the remaining elements of $\alpha$. With this notation, we have ${\underbeta}=\left[\begin{smallmatrix}{\underbeta}_0\\ \mathbf{0}\end{smallmatrix}\right]$.

Let ${\underbeta}$ be the true regressor and let ${{\hatunderbeta}}$ be an estimate of ${\underbeta}$ obtained from the training data. Let $\underline{w}\defeq \hat{{\underbeta}}-{\underbeta}$. According to our model setting, the expected test error satisfies
\begin{align}\label{eq.testError_w}
    &\mathds{E}_{x,\ea}\left[\left(x^T{{\hatunderbeta}}-(x^T{\underbeta}+\ea)\right)^2\right]
    =\mathds{E}_{x,\ea}\left[(x^T \underline{w}-\ea)^2\right]
    =\|\underline{w}\|_2^2+\sigma^2.
\end{align}
Since $\sigma^2$ is given, in the rest of the paper we will mostly focus
on the model error $\|\underline{w}\|_2$. Note that if ${{\hatunderbeta}}=0$, we have $\|\underw\|_2^2=\|{\underbeta}\|_2^2$, which is also the average strength of the signal $x^T
{\underbeta}$ and is referred to as the ``null risk'' in
\citep{hastie2019surprises}. We define the signal-to-noise ratio as $\SNR\defeq \|\underbeta\|_2^2/\sigma^2$.

We next describe how BP computes ${{\hatunderbeta}}$ from training data $(\Xtr$,  $\Ytr)$, where $\Xtr\in\mathds{R}^{n\times p}$ and $\Ytr\in\mathds{R}^n$. For ease of analysis, we normalize each column of $\Xtr$ as follows.
Assume that we have $n$ \emph{i.i.d.} samples in the form of Eq.~(\ref{eq.model}). For each sample $k$, we first divide both sides of Eq. (\ref{eq.model}) by $\sqrt{n}$, i.e.,
\begin{align}\label{eq.normalize_n}
    \frac{y_k}{\sqrt{n}}=\left(\frac{x_k}{\sqrt{n}}\right)^T{\underbeta} + \frac{\ea_k}{\sqrt{n}}.
\end{align}

We then form a matrix $\mathbf{H}\in\mathds{R}^{n\times p}$ so that each row $k$ is the sample $x_k^T$. Writing $\mathbf{H}=[\mathbf{H}_1\ \mathbf{H}_2\ \cdots\ \mathbf{H}_p]$,
we then have $\mathds{E}[\|\mathbf{H}_i\|_2^2]=n$, for each column $ i\in\{1,2,\cdots,p\}$. Now, let $\Xtr=[\X_1\ \X_2\ \cdots\ \X_p]$ be normalized in such a way that
\begin{align}\label{def.norm_X}
    \X_i=\frac{\mathbf{H}_i}{\|\mathbf{H}_i\|_2},\ \forall i\in\{1,2,\cdots,p\},
\end{align}
and let each row $k$ of $\Ytr$ and $\etr$ be the corresponding values of $y_k/\sqrt{n}$ and $\ea_k/\sqrt{n}$ of the sample. Then, each column $\X_i$ will have a unit $\ell_2$-norm. We can then write the training data as
\begin{align}\label{eq.data}
    \Ytr= \Xtr\beta+\etr,
\end{align}
where
\begin{align}
    \ith{\beta} = \frac{\|\mathbf{H}_i\|_2}{\sqrt{n}}\ith{\underbeta}\ \forall i\in \{1,2,\cdots, p\}.\label{eq.scale_beta}
\end{align}
Note that the above normalization of $\Xtr$ leads to a small distortion
of the ground truth from ${\underbeta}$ to $\beta$, but it eases our
subsequent analysis. When $n$ is large, the distortion is small, which will be made precise below (in Lemma~\ref{lemma:distortion}). Further, we have $\mathds{E}[\|\etr\|_2^2]=\sigma^2$.


In the rest of this paper, we focus on the situation of overparameterization, i.e., $p>n$. Among many different estimators of $\beta$, we are interested in those that perfectly fit the training data, i.e., 
\begin{align}\label{eq.BP_constraint}
    \Xtr\hat{\beta}=\Ytr.
\end{align}
When $p>n$, there are infinitely many $\hat{\beta}$'s that satisfy Eq. \eqref{eq.BP_constraint}. In BP \citep{chen2001atomic}, $\hat{\beta}$ is chosen by solving the following problem
\begin{align}\label{eq.BP_origin}
    \min_{\tilde{\beta}}\|\tilde{\beta}\|_1,\ \st \Xtr\tilde{\beta}=\Ytr.
\end{align}
In other words, given $\Xtr$ and $\Ytr$, BP finds the overfitting solution with the minimal $\ell_1$-norm. Note that as long as $\Xtr$ has full row-rank (which occurs almost surely), Eq.~\eqref{eq.BP_origin} always has a solution. Further, when Eq.~\eqref{eq.BP_origin} has one or multiple solutions, there must exist one with at most $n$ non-zero elements (we prove this fact as Lemma~\ref{le.solution_at_most_n_non_zero} in Appendix~\ref{app:proof_n_non_zero} of Supplementary Material).
We thus use $\betaBP$ to denote any such solution with at most $n$ non-zero elements.
Define $\wBP\defeq \betaBP-\beta$. In the rest of our paper, we will show how to estimate the model error $\|\wBP\|_2$ of BP as a function of the system parameters such as $n$, $p$, $s$, and $\sigma^2$. 
Note that before we apply the solution of BP in Eq.~(\ref{eq.BP_origin}) to new test data, we should re-scale it back to (see Eq.~(\ref{eq.scale_beta}))
\begin{align*}
    \ith{\underbetaBP}=\frac{\sqrt{n}\ith{\betaBP}}{\|\mathbf{H}_i\|_2},\ \forall i\in \{1,2,\cdots, p\},
\end{align*}
and measure the generalization performance by the unscaled version of $\wBP$, i.e.,  ${\underwBP}={\underbetaBP}-{\underbeta}$. The following lemma shows that, when $n$ is large, the difference between $\wBP$ and ${\underwBP}$ in term of either $\ell_1$-norm or $\ell_2$-norm is within a factor of $\sqrt{2}$ with high probability.
\begin{lemma}
\label{lemma:distortion}
For both $d=1$ and $d=2$, we have
\begin{align*}
    &\prob{\|{\underwBP}\|_d\leq \sqrt{2}\|\wBP\|_d}\geq 1-\exp\left(-\frac{n}{16}+\ln (2n)\right),\\
    &\prob{\|\wBP\|_d\leq \sqrt{2}\|{\underwBP}\|_d}\geq 1-\exp\left(-\frac{2-\sqrt{3}}{2}n+\ln (2n)\right).
\end{align*}
\end{lemma}
Therefore, in the rest of the paper we will focus on bounding $\|\wBP\|_2$.

Note that the model error of BP was also studied in \citep{donoho2005stable}.
However, the notion of BP therein is different in that the
estimator $\hat{\beta}$ only needs to satisfy
$\|\Ytr-\Xtr\hat{\beta}\|_2\leq\delta$. The main result 
(Theorem 3.1) of \citep{donoho2005stable} requires $\delta$ to be greater
than the noise level $\|\etr\|_2$,
and thus cannot be zero.  
(An earlier version \cite{donohostable} of \citep{donoho2005stable}
incorrectly claimed that $\delta$ can be 0 when $\|\etr\|_2>0$, which is
later corrected in \cite{donoho2005stable}.) Therefore, the result of
\citep{donoho2005stable} does not capture the performance of BP for the
overfitting setting of Eq.~(\ref{eq.BP_constraint}).
Similarly, the analysis of BP in \cite{donoho2001uncertainty} assumes no
observation noise, which is also different from
Eq.~(\ref{eq.BP_constraint}).

\section{Main results}\label{sec.main_results}
Our main result is the following upper bound on the model error of BP with overfitting.
\begin{theorem}[Upper Bound on $\|\wBP\|_2$]\label{th.main}
When $s\leq \sqrt{\frac{n}{7168\ln (16n)}}$, if $p\in \left[(16n)^4,\ \exp\left(\frac{n}{1792s^2}\right)\right]$, then
with probability at least $1-6/p$, we have
\begin{align}\label{eq.main_bound}
    \frac{\|\wBP\|_2}{\|\etr\|_2}\leq 2+8\left(\frac{7n}{\ln p}\right)^{1/4}.
\end{align}
\end{theorem}

It is well known that the model error of the minimum MSE (mean-square-error) solution has a peak when $p$ approaches $n$ from below \cite{belkin2019two}. Further, when $p=n$, $\betaBP$ coincides with the min-MSE solution with high probability. Thus, Theorem~\ref{th.main} shows that the model error of overfitting BP must decrease from that peak (i.e., may exhibit double descent) when $p$ increases beyond $n$, up to a value exponential in $n$.
Note that the assumption $s\leq \sqrt{\frac{n}{7168\ln (16n)}}$, which states that the true model is sufficiently sparse\footnote{Such sparsity requirements are not uncommon, e.g., Theorem~3.1 of \cite{donoho2005stable} also requires the sparsity to be below some function of $M$ (incoherence of $\mathbf{X}_{\text{train}}$), which is related to $n$ according to Proposition~\ref{prop.M} in our paper.}, implies that the interval $\left[(16n)^4,\ \exp\left(\frac{n}{1792s^2}\right)\right]$ is not empty. We note that the constants in Theorem~\ref{th.main} may be loose (e.g., the values of $n$ and $p$ need to be quite large for $p$ to fall into the above interval). These large constants are partly due to our goal to obtain high-probability results, and they could be further optimized. Nonetheless, our numerical results below suggest that the predicted trends (e.g., for double descent) hold for much smaller $n$ and $p$.

The upper limit of $p$ in Theorem \ref{th.main} suggests a descent floor
for BP, which is stated below. 
Please refer to Supplementary Material (Appendix \ref{app.proof_s}) for
the proof.
\begin{corollary}\label{coro.only_s}
If $1\leq s\leq \sqrt{\frac{n}{7168\ln (16n)}}$, then by setting $p=\left\lfloor\exp\left(\frac{n}{1792s^2}\right)\right\rfloor$, we have
\begin{align}\label{eq.main_bound_s}
    \frac{\|\wBP\|_2}{\|\etr\|_2}\leq 2+32\sqrt{14}\sqrt{s}
\end{align}
with probability at least $1-6/p$.
\end{corollary}

To the best of our knowledge, 
Theorem~\ref{th.main} and Corollary~\ref{coro.only_s} are the first in
the literature to quantify the double-descent of overfitting BP for
finite $n$ and $p$. Although these results mainly focus on large $p$, they reveal several important insights,
highlighting the significant differences (despite some similarity)
between the
double-descent of BP and min $\ell_2$-norm solutions. (The codes for the following numerical experiments can be found in our Github page\footnote{\href{https://github.com/functionadvanced/basis_pursuit_code}{https://github.com/functionadvanced/basis$\_$pursuit$\_$code}}.)

\begin{figure}
    \centering
    \includegraphics[width=5in]{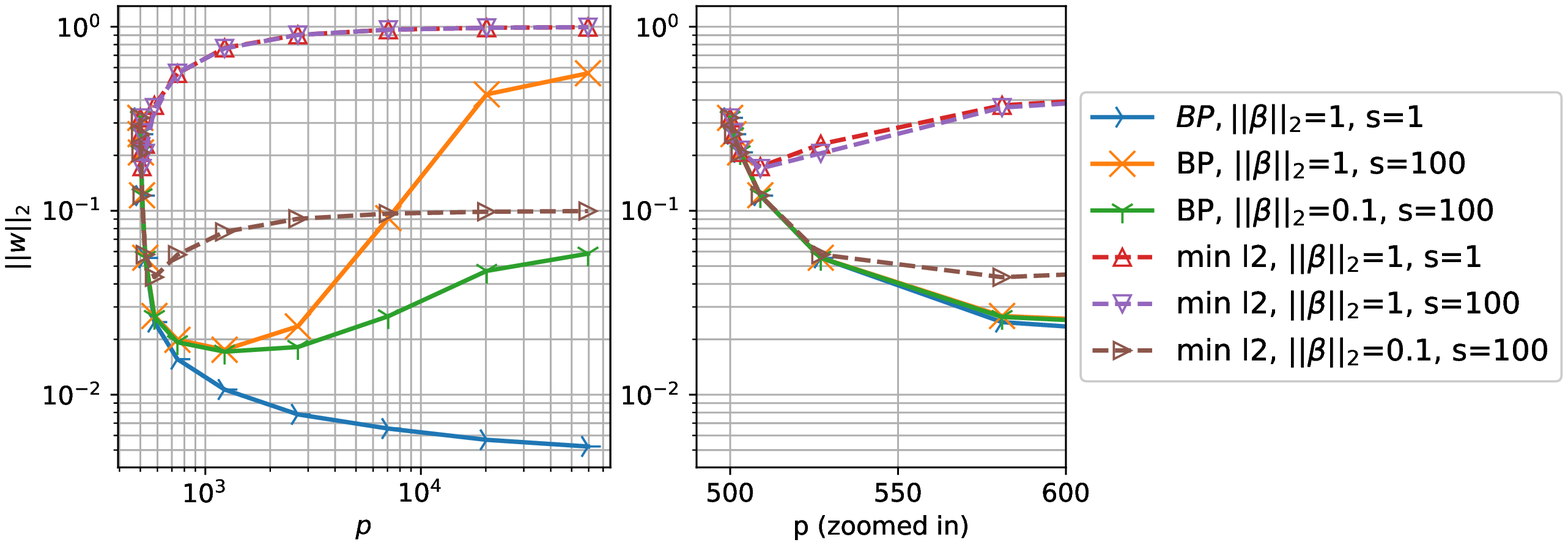}
        \caption{Compare BP with min $\ell_2$-norm for different values of $\|\beta\|_2$ and $s$, where
	$\|\etr\|_2=0.01$, $ n=500$. The right figure is a zoomed-in
	version of the left figure for $p \in [500,600]$.}
        \label{fig.compare_BP_minL2}
\end{figure}
    
\textbf{(i) The double-descent upper-bound of BP is independent of $\|\beta\|_2$. In
contrast, the descent of min $\ell_2$-norm solutions is raised by
		$\|\beta\|_2$.} 
The upper bounds in both Theorem~\ref{th.main} and Corollary~\ref{coro.only_s} 
	do not depend on the signal strength $\|\beta\|_2$. For min
		$\ell_2$-norm overfitting solutions, however, if we let $\wltwo$
		denote the model error, for comparable Gaussian models we can
		obtain (see, e.g., Theorem~2 of \cite{belkin2019two}):
\begin{eqnarray}
\label{eqn:double-descent:l2}
	\ex[\|\wltwo\|_2^2] &=& \|\beta\|_2^2 \left(1-\frac{n}{p}\right) 
		+ \frac{\sigma^2 n}{p-n-1}, \mbox{ for $p \ge n+2$.}
\end{eqnarray}
To compare with Eq.~\eqref{eq.main_bound}, recall that $\mathds{E}[\|\etr\|_2^2]=\sigma^2$. (In fact, we can show that, when $n$ is large, $\|\etr\|_2$ is close to $\sigma$ with high probability. See Supplementary Material (Appendix~\ref{ap.sec.noise}).) From Eq.~\eqref{eqn:double-descent:l2}, we can see that the right-hand-side increases
		with $\|\beta\|_2$. This difference between Eq.~\eqref{eq.main_bound} and Eq.~\eqref{eqn:double-descent:l2} is confirmed by our
		numerical results in
		Fig.~\ref{fig.compare_BP_minL2}\footnote{The
		direction of $\beta$ is chosen uniformly at random in all
		simulations.}:
		When we compare $\|\beta\|_2=1$ and $\|\beta\|_2=0.1$
		(both with $s=100$), the descent of min $\ell_2$-norm
		solutions (dashed $\triangledown$ and $\rhd$) varies significantly with $\|\beta\|_2$, while the
		descent of BP (solid $\times$ and $\Ydown$) does not vary much with $\|\beta\|_2$ within the region where the model error decreases with $p$.
Intuitively, when the signal strength is strong, it should be easier to
		detect the true model. The deterioration with
		$\|\beta\|_2$ suggests that some of the signal is
		``spilled'' into the model error of min $\ell_2$-norm
		solutions, which does not occur for BP in the descent
		region.

\emph{Remark:} Although the term ``double descent'' is often used in the literature \cite{belkin2019two}, as one can see from Fig.~\ref{fig.compare_BP_minL2}, the model error will eventually increase again when $p$ is very large. For BP, the reason is that, when $p$ is very large, there will be some columns of $\mathbf{X}_{\text{train}}$ very similar to those of true features. Then, BP will pick those very similar but wrong features, and the error will approach the null risk. As a result, the whole double-descent behavior becomes ``descent-ascent($p<n$)-descent($p>n$)-ascent''. Note that similar behavior also arises for the model error of min $\ell_2$-norm overfitting solutions (see Fig.~\ref{fig.compare_BP_minL2}). Throughout this paper, we will use the term ``descent region'' to refer to the second descent part ($p>n$), before the model error ascends for much larger $p$.

\textbf{(ii) The descent of BP is slower but easier to observe.}
The right-hand-side of Eq. (\ref{eq.main_bound}) is inversely
proportional to $\ln^{1/4} p$, which is slower than the
inverse-in-$p$ descent in Eq.~(\ref{eqn:double-descent:l2}) for min $\ell_2$-norm solutions. This
slower descent can again be verified by the numerical result in
Fig.~\ref{fig.compare_BP_minL2}. On the other hand, this property also
means that the descent of BP is much easier to observe as it holds for an
exponentially-large range of $p$, which is in sharp contrast to min
$\ell_2$-norm solutions, whose descent quickly stops and bounces back to the
null risk (see Fig.~\ref{fig.compare_BP_minL2}).

%

Readers may ask whether the slow descent in $\ln p$ is fundamental to
BP, or is just because Theorem~\ref{th.main} is an upper bound. In fact,
we can obtain the following (loose) lower bound:

\begin{proposition}[lower bound on $\|\wBP\|_2$]
\label{prop:lower:BP}
When $p\leq e^{(n-1)/16}/n$, $n \ge s$, and $n\geq 17$, we have
\begin{align*}
	\frac{\|\wBP\|_2}{\|\etr\|_2} 
	\geq \frac{1}{3\sqrt{2}}\sqrt{\frac{1}{\ln p}}
\end{align*}
with probability at least $1-3/n$.
\end{proposition}
Although there is still a significant gap from the upper bound in
Theorem~\ref{th.main}, Proposition~\ref{prop:lower:BP} does imply that the
descent of BP cannot be faster than $1/\sqrt{\ln p}$. The proof of
Proposition~\ref{prop:lower:BP} is based on a lower-bound on
$\|\wBP\|_1$, which is actually tight.  For details, please refer to
Supplementary Material (Appendix~\ref{app:lower_bounds}).

\begin{figure}
    
\end{figure}

\begin{figure}
    \centering
    \begin{minipage}[t]{0.47\textwidth}
        \centering
        \includegraphics[width=2.5in]{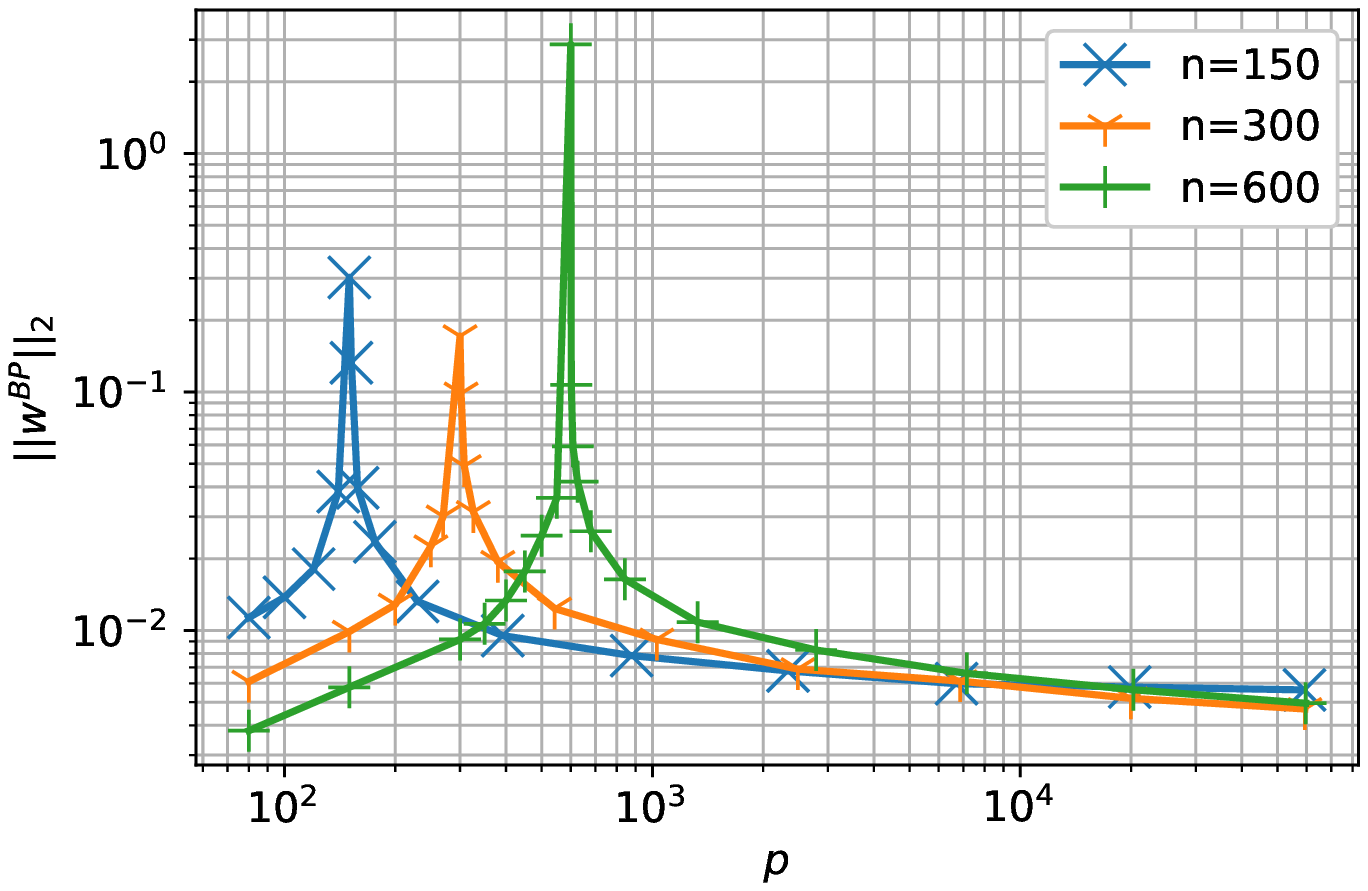}
        \caption{Curves of $\|\wBP\|_2$ with different $n$, where $\|\beta\|_2=1$, $\|\etr\|_2=0.01$, $s=1$. (Note that for $p<n$ we report the model error of the min-MSE solutions.)}
        \label{fig.change_n}
    \end{minipage}\hfill
    \begin{minipage}[t]{0.47\textwidth}
        \centering
    \includegraphics[width=2.2in]{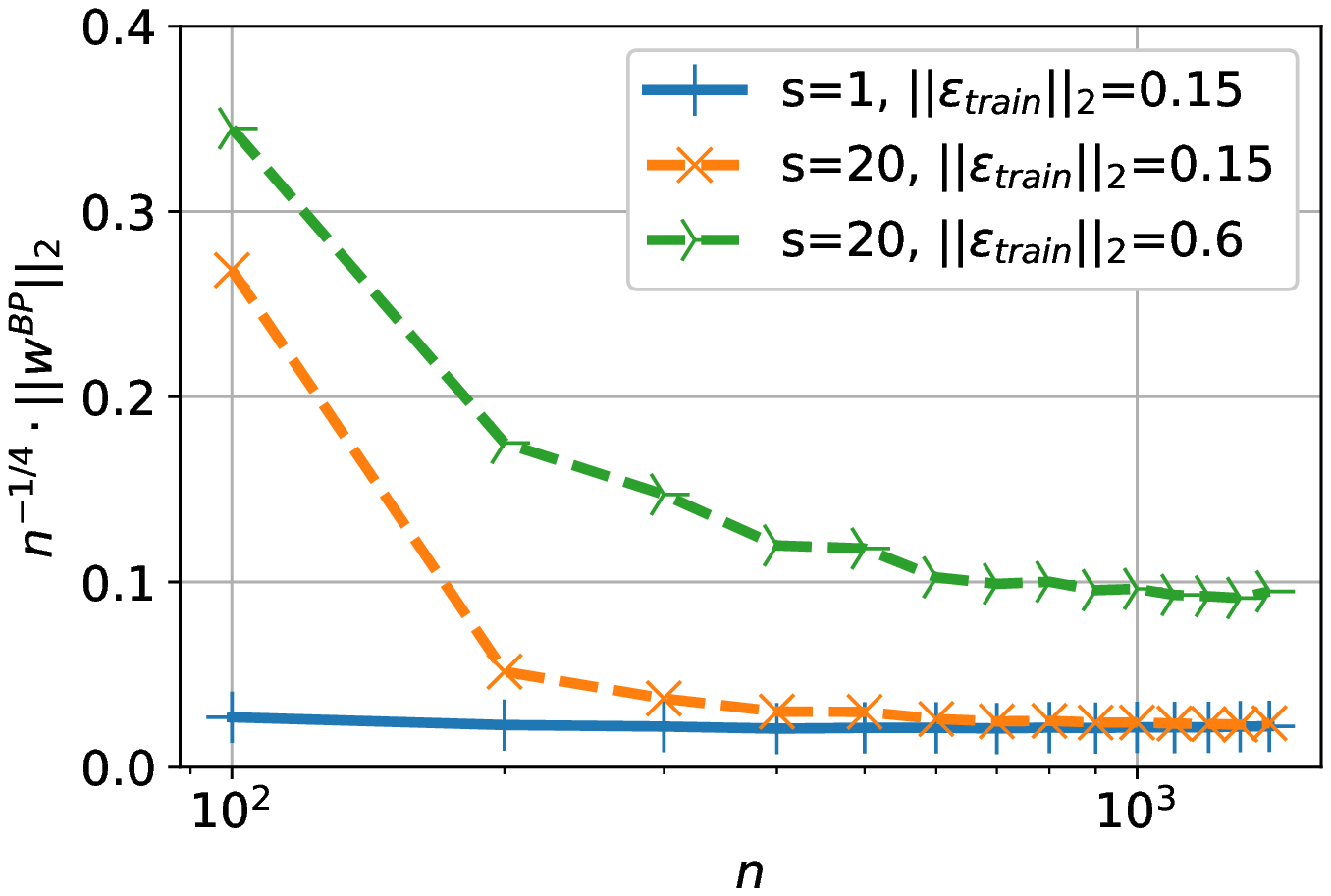}
    \caption{Curves of $n^{-1/4}\|\wBP\|_2$, where $p=5000$, $\|\beta\|_2=1$.}
    \label{fig.validate_n}
    \end{minipage}
\end{figure}

\textbf{(iii) Large $n$ increases the model error of BP, but large $p$
compensates it in such a way that the descent floor is independent of
$n$.} From Eq.~(\ref{eq.main_bound}), we can observe that, for a fixed $p$
inside the descent region, larger $n$ makes the performance of BP worse.
This is
confirmed by our numerical results in Fig.~\ref{fig.change_n}: 
when $p$ is relatively small but larger than $n$ (i.e., in the descent region), 
$\|\wBP\|_2$ is larger when $n$ increases from $n=150$ to $n=600$. (Note that when $p<n$, with high probability overfitting cannot happen, i.e., no solution can satisfy Eq.~\eqref{eq.BP_constraint}. Instead, for $p<n$ we plot the model error of the min-MSE solutions, which ascends to the peak at $p=n$ right before the descent region.)
%
While surprising, this behavior is however
reasonable. As $n$ increases, the null space corresponding to Eq.
(\ref{eq.BP_constraint}) becomes smaller. Therefore, more data means
that BP has to ``work harder" to fit the noise in those data, and
consequently BP introduces larger model errors.  On the positive side,
while larger $n$ degrades the performance of BP,
the larger $p$ (i.e.,
overparameterization) helps in a non-trivial way that cancels out the
additional increase in model error, so that ultimately the descent floor
is independent of $n$ (see Corollary~\ref{coro.only_s}). This is again
confirmed by Fig.~\ref{fig.change_n} where the lowest points of the overfitting regime ($p>n$) of all curves for
different $n$ are very close.
We note that
for the min $\ell_2$-norm solution, one can easily verify that its descent
in Eq.~(\ref{eqn:double-descent:l2}) stops at $p =
\frac{(n+1)\|\beta\|_2}{\|\beta\|_2 - \sigma}$, which leads to a descent
floor for $\|\wltwo\|_2$ at the level around $\sqrt{2 \|\beta\|_2 \sigma -
\sigma^2}$, independently of $n$. Further, for $p$ smaller
than the above value, its model error will also increase with $n$. Thus,
from this aspect both BP and min $\ell_2$-norm solutions have similar
dependency on $n$. However, as we comment below, the descent floor of BP
can be much wider and deeper.

\textbf{(iv) For high $\SNR$ and small $s$, the descent floor of BP is
lower and wider.} Comparing the above-mentioned descent floors between BP and min
$\ell_2$-norm solutions, we can see, when $\SNR=\|\beta\|_2^2/\sigma^2
\gg 1$, the
descent floor of BP is about $\Theta(\sqrt{\frac{1}{s}\sqrt{\SNR}})$
lower than that of min $\ell_2$-norm solutions. Further, since the
model error of BP decreases in $\ln p$, we expect the descent floor
of BP to be significantly wider than that of min $\ell_2$-norm
solutions. Both are confirmed numerically in Fig.
\ref{fig.compare_BP_minL2}, suggesting that BP may produce more
accurate solutions more easily for higher-$\SNR$ training data and
sparser models. 

Finally, Fig.~\ref{fig.validate_n} verifies that the quantitative
dependency of the model error on $n$, $s$ and $\|\etr\|_2$ predicted by
our upper bound in Theorem~\ref{th.main} is actually quite tight.  For
large $n$, we can see that all 
curves of $n^{-1/4}\|\wBP\|_2$ approach horizontal lines, suggesting
that 
$\|\wBP\|_2$ is indeed proportional to $n^{1/4}$ in the descent region.
(The deviation at small $n$ is because, 
for such a small $n$, $p=5000$ has already passed the descent region.)
Further, we observe that the horizontal part of the blue curve
``$s=1,\|\etr\|_2=0.15$'' almost overlaps with that of the orange curve
``$s=20,\|\etr\|_2=0.15$'', which matches with
Theorem~\ref{th.main} that the descent is independent of $s$
in the descent region. 
Finally, the horizontal part
of the green curve ``$s=20,\|\etr\|_2=0.6$'' are almost 4 times as high
as that
of the orange curve ``$s=20,\|\etr\|_2=0.15$'', which matches with
Theorem~\ref{th.main} that the descent is proportional to
$\|\etr\|_2$.

We note that with a careful choice of regularization parameters, LASSO
can potentially drive the $\ell_2$-norm of the model error to be as low as 
$\Theta(\sigma \sqrt{s \log p / n})$ \cite{meinshausen2009lasso}. As we
have shown above, neither
BP nor min $\ell_2$-norm solutions can push the model error to be this
low. However, LASSO by itself does not produce overfitting solutions,
and the type of explicit regularization in LASSO (to avoid
overfitting) is also not used in DNN. Thus, it remains a puzzle
whether and how one can design practical overfitting solution that can
reach this level of accuracy without any explicit regularization.

\section{Main ideas of the proof}
\label{sect:proof}

In this section, we present the main ideas behind the proof of Theorem \ref{th.main}, which also reveal additional insights for BP. We start with the following definition.
Let $w^I$ be the solution to the following problem (recall that $w_0$ denotes the sub-vector that consists of the first $s$ elements of $w$, i.e., corresponding to the non-zero elements of the true regressor $\underline{\beta}$):
\begin{align}\label{eq.def_WI}
    \min_w\|w\|_1,\ \st \Xtr w=\etr,\ w_0=\mathbf{0}.
\end{align}
In other words, $w^I$ is the regressor that fits only the noise $\etr$.
Assuming that the matrix $[\X_{s+1}\ \X_{s+2}\ \cdots\ \X_p]\in\mathds{R}^{n\times(p-s)}$ has full row-rank (which occurs almost surely), $w^I$ exists if and only if $p-s\geq n$, which is slightly stricter than $p > n$. The rest of this paper is based on the condition that $w^I$ exists, i.e., $p-s\geq n$. Notice that in Theorem \ref{th.main}, the condition $p\geq (16n)^4$ already implies that $p\geq (16n)^4\geq 2n\geq s+n$.

The first step (Proposition \ref{prop.bound_WB1_KM} below) is to relate
the magnitude of $\wBP$ with the magnitude of $w^I$. The reason that we
are interested in this relationship is as follows. Note that one
potential way for an overfitting solution to have a small model error is
that the solution uses the $(p-s)$ ``redundant" elements of the
regressor to fit the noise, without distorting the $s$ ``significant"
elements (that correspond to the non-zero basis of the true regressor).
In that case, as $(p-s)$ increases, it will be increasingly easier for the
``redundant" elements of the regressor to fit the noise, and thus the
model error may improve with respect to $p$. In other words, we expect
that $\|w^I\|_1$ will decrease as $p$ increases. However, it is not
always true that, as the ``redundant" elements of the regressor
fit the noise better, they do not distort the ``significant" elements of the
regressor. Indeed, $\ell_2$-minimization would be such a
counter-example: as $p$ increases, although it is also increasingly easier for
the regressor to fit the noise \citep{muthukumar2019harmless}, the
``significant" elements of the regressor also go to zero
\citep{belkin2019two}. This is precisely the reason why the model error
of the min $\ell_2$-norm overfitting solution 
in Eq.~(\ref{eqn:double-descent:l2}) 
quickly approaches the null risk
$\|\beta\|_2$ as $p$ increases. 
In contrast, Proposition \ref{prop.bound_WB1_KM} below shows
that this type of undesirable distortion will not occur for BP under
suitable conditions.

Specifically, define the \emph{incoherence} of $\Xtr$
\cite{donohostable,donoho2005stable} as 
\begin{align}\label{eq.def_M}
    &M\defeq \max_{i\neq j}\left|\X_i^T\X_j\right|,
\end{align}
where $\X_i$ and $\X_j$ denote $i$-th and $j$-th columns of $\Xtr$,
respectively. Thus, $M$ represents the largest absolute value of
correlation (i.e., inner-product) between any two columns of $\Xtr$
(recall that the $\ell_2$-norm of each column is exactly
$1$).
%
Further, let
\begin{align}\label{eq.def_K}
    K\defeq\frac{1+M}{sM}-4.
\end{align}
We then have the following proposition that relates the model error $\wBP$ to the magnitude of $w^I$.
\begin{proposition}\label{prop.bound_WB1_KM}
When $K>0$, we have
\begin{align}\label{eq.prop1}
    \|\wBP\|_1\leq \left(1+\frac{8}{K}+2\sqrt{\frac{1}{K}}\right)\|w^I\|_1+\frac{2\|\etr\|_2}{\sqrt{KM}}.
\end{align}
\end{proposition}
Please refer to Supplementary Material (Appendix~\ref{app:proof:wb1})
for the proof.  Proposition \ref{prop.bound_WB1_KM} shows that, as long
as $\|w^I\|_1$ is small, $\|\wBP\|_1$ will also be small. Note that in
Eq. \eqref{eq.def_M}, $M$ indicates how similar any two features
(corresponding to two columns of $\Xtr$) are. As long as $M$ is much
smaller than $1/s$, in particular if $M\leq 1/(8s)$, then the value of
$K$ defined in Eq. \eqref{eq.def_K} will be no smaller than $4$. Then,
the first term of Eq. \eqref{eq.prop1} will be at most a constant
multiple of $\|w^I\|_1$. In conclusion, $\|\wBP\|_1$ will not be much
larger than $\|w^I\|_1$ as long as the columns of $\Xtr$ are not very
similar.

Proposition \ref{prop.bound_WB1_KM} only captures the $\ell_1$-norm of $\wBP$. Instead, the test error in Eq. \eqref{eq.testError_w} is directly related to the $\ell_2$-norm of $\wBP$. Proposition \ref{prop.wB2_wB1} below relates $\|\wBP\|_2$ to $\|\wBP\|_1$.

\begin{proposition}\label{prop.wB2_wB1}
The following holds:
\begin{align*}
    \|\wBP\|_2\leq\|\etr\|_2+\sqrt{M}\|\wBP\|_1.
\end{align*}
\end{proposition}

The proof is available in Supplementary Material
(Appendix~\ref{app:proof:wb2}). 
Note that for an arbitrary vector $\alpha\in\mathds{R}^p$, we can only infer $\|\alpha\|_2\leq \|\alpha\|_1$. In contrast, Proposition \ref{prop.wB2_wB1} provides a much tighter bound for $\|\wBP\|_2$ when $M$ is small (i.e., the similarity between the columns of $\Xtr$ is low).

Combining Propositions \ref{prop.bound_WB1_KM} and \ref{prop.wB2_wB1}, we have the following corollary that relates $\|\wBP\|_2$ to $\|w^I\|_1$.
\begin{corollary}\label{coro.wB2_K}
When $K>0$, we must have
\begin{align*}
    \|\wBP\|_2\leq& \big(1+\frac{2}{\sqrt{K}}\big)\|\etr\|_2+\sqrt{M}\big(1+\frac{8}{K}+\frac{2}{\sqrt{K}}\big)\|w^I\|_1.
\end{align*}
\end{corollary}

It remains to bound $\|w^I\|_1$ and $M$. The following proposition gives an upper bound on $\|w^I\|_1$.
\begin{proposition}\label{prop.new_WI}
When $n\geq 100$ and $p\geq (16n)^4$, with probability at least $1-2e^{-n/4}$ we have
\begin{align}\label{eq.WI}
    \frac{\|w^I\|_1}{\|\etr\|_2}\leq \sqrt{1+\frac{3n/2}{\ln p}}.
\end{align}
\end{proposition}

The proof of Proposition \ref{prop.new_WI} is quite involved and is 
available in Supplementary Material (Appendix~\ref{app:proof:wI}).
Proposition \ref{prop.new_WI} shows that $\|w^I\|_1$ decreases in $p$ at
the rate of $O(\sqrt{n/\ln p})$. This is also the reason that $n/\ln p$
shows up in the upper bound in Theorem \ref{th.main}. Further,
$\|w^I\|_1$ is upper bounded by a value proportional to $\|\etr\|_2$, which, when combined with
Corollary \ref{coro.wB2_K}, implies that $\|\wBP\|_2$ is on the order of
$\|\etr\|_2$. Note that the decrease of $\|w^I\|$ in $p$ trivially follows from its
definition in Eq.~\eqref{eq.def_WI} because, when $w^I$ contains
more elements, the optimal $w^I$ in Eq.~\eqref{eq.def_WI} should only have a
smaller norm. In contrast, the contribution of Proposition \ref{prop.new_WI} is in
capturing the exact speed with which $\|w^I\|_1$ decreases with $p$,
which has not been studied in the literature. 
When $p$ approaches
$+\infty$, the upper bound in Eq.~\eqref{eq.WI} becomes $1$. Intuitively,
this is because with an infinite number of features, eventually there are
columns of $\Xtr$ that are very close to the direction of $\etr$. By choosing
those columns, $\|w^I\|_1$ approaches $\|\etr\|_2$. 
Finally, the upper bound in Eq.~\eqref{eq.WI} increases with the number of
samples~$n$. As we discussed earlier, this is because as $n$ increases,
there are more constraints in Eq.~\eqref{eq.def_WI} for $w^I$ to fit. Thus,
the magnitude of $w^I$ increases.

Next, we present an upper bound on $M$ as follows.
\begin{proposition}\label{prop.M}
When $p\leq e^{n/36}$, with probability at least $1-2e^{-\ln p}-2e^{-n/144}$ we have
\begin{align*}
    M\leq 2\sqrt{7}\sqrt{\frac{\ln p}{n}}.
\end{align*}
\end{proposition}
The proof is available in Supplementary Material (Appendix~\ref{app.M}).
To understand the intuition behind, note that it is not hard to verify that for any $i\neq j$, the standard deviation of $\X_i^T\X_j$ is approximately equals to $1/\sqrt{n}$.
Since $M$ defined in Eq. \eqref{eq.def_M} denotes the maximum over $p\times(p-1)$ such
pairs of columns, it will grow faster than $1/\sqrt{n}$.
Proposition \ref{prop.M} shows that the additional multiplication factor
is of the order $\sqrt{\ln p}$.
%
As $p$ increases, eventually we can find some columns that are close to each other, which implies that $M$ is large. When some columns among the last $(p-s)$ columns of $\Xtr$ are quite similar to the first $s$ columns, $M$ will be large and BP cannot distinguish the true features from spurious features. This is the main reason that the ``double descent" will eventually stop when $p$ is very large, and thus Theorem \ref{th.main} only holds up to a limit of $p$. 

Combining Propositions \ref{prop.new_WI} and \ref{prop.M}, we can then prove Theorem~\ref{th.main}. Please
refer to Supplementary Material (Appendix~\ref{app:proof:main}) for details. 


\section{Conclusions and future work}
In this paper, we studied the generalization power of basis pursuit (BP)
in the overparameterized regime when the model overfits the data. Under
a sparse linear model with \emph{i.i.d.} Gaussian features, we showed that
the model error of BP exhibits ``double descent" in a way quite
different from min $\ell_2$-norm solutions. Specifically, the
double-descent upper-bound of BP is independent of the signal strength. Further, for high
$\SNR$ and sparse models, the descent-floor of BP can be much lower and
wider than that of min $\ell_2$-norm solutions. 

There are several interesting directions for future work. First, the gap between our upper bound
(Theorem~\ref{th.main}) and lower bound (Proposition~\ref{prop:lower:BP}) is still large, which suggests rooms to tighten these bounds.
Second, these bounds work when $p$ is much larger than $n$, e.g., $p\sim \exp(n)$. It would be useful to also understand the error bound of BP when $p$ is just a little larger than $n$ (similar to \cite{mitra2019understanding} but for finite $p$ and $n$).
Third, we only study
isotropic Gaussian features in this paper.
It would be important to see if our main conclusions can also be
generalized to other feature models (e.g., Fourier features
\citep{rahimi2008random}), models with mis-specified features
\citep{belkin2019two, hastie2019surprises,mitra2019understanding}, or even the 2-layer neural
network models of \citep{mei2018mean}. 
Finally, we hope that the difference between min $\ell_1$-norm solutions and min $\ell_2$-norm solutions reported here could help us understand the generalization power of overparameterized DNNs, or lead to training methods for DNNs with even better performance in such regimes.

\section*{Acknowledgement}
This work has been supported in part by an NSF sub-award via Duke University (NSF IIS-1932630), NSF grants CAREER CNS-1943226, ECCS-1818791, CCF-1758736, CNS-1758757, CNS-1717493, ONR grant N00014-17-1-2417, and a Google Faculty Research Award.

\section*{Broader Impact}

Understanding the generalization power of heavily over-parameterized
networks is one of the most foundation aspects of deep learning. By
understanding the double descent of generalization errors for Basis Pursuit
(BP) when overfitting occurs, our work advances the understanding
of how superior generalization power can arise for overfitting
solutions. Such an understanding will contribute to laying a solid
theoretical foundation of deep learning, which in turn may lead to
practical guidelines in applying deep learning in diverse fields such as 
image processing and natural languages processing.

Controlling the $\ell_1$-norm also plays a fundamental role in optimization with
sparse models, which have found important applications in, e.g.,
compressive sensing and matrix completion. Without thoroughly studying
the overparameterized regime of BP (which minimizes $\ell_1$-norm while
overfitting the data), the theory of double descent remains incomplete.

The insights revealed via our proof techniques in Section 4 (i.e., the
relationship between the error of fitting data and the error of fitting
only noise) not only help to analyze the double descent of BP, but may
also be of value to other more general models. 

Potential negative impacts: Our theoretical results in this paper are
the first step towards understanding the double descent of BP, and
should be used with care. In particular, there is still a significant
gap between our upper and lower bounds. Thus, the significance of
our theories lies more in revealing the general trend rather than
precise characterization of double descent.  Further, the Gaussian model
may also be a limiting factor. Therefore, more efforts will be needed to
sharpen the bounds and generalize the results for applications in
practice.

\bibliographystyle{abbrv}



\newpage
\appendix

\section{Proof of Lemma~\ref{le.solution_at_most_n_non_zero}}\label{app:proof_n_non_zero}
\begin{lemma}\label{le.solution_at_most_n_non_zero}
If Eq.~\eqref{eq.BP_origin} has one or multiple solutions, there must exist one with at most $n$ non-zero elements.
\end{lemma}
\begin{proof}
We prove by contradiction. Suppose on the contrary that every solution to Eq.~\eqref{eq.BP_origin} has at least $(n+1)$ non-zero elements. Let $\beta_0$ denote a solution with the smallest number of non-zero elements. Let $\mathcal{A}$ denote the set of indices of non-zero elements of $\beta_0$. Then, we have $|\mathcal{A}|\geq n+1$. Below, we will show that there must exist another solution to Eq.~\eqref{eq.BP_origin} with strictly fewer non-zero elements than $\beta_0$, which leads to a contradiction. Towards this end, note that since $\Xtr$ has only $n$ rows, the subset of columns $\X_i$, $i\in \mathcal{A}$, must be linear dependent. Therefore, we can always find a non-empty set $\mathcal{B}\in\mathcal{A}$ and coefficients $c_i\neq 0$ for $i\in\mathcal{B}$ such that
\begin{align}\label{eq.temp_101101}
    \sum_{i\in\mathcal{B}} c_i\X_i=\bm{0}.
\end{align}
Define $\beta_\lambda\in\mathds{R}^p$ for $\lambda\in\mathds{R}$ such that
\begin{align*}
    \beta_\lambda[i]=\begin{cases}
    \beta_0[i]+\lambda c_i,&\text{ if }i\in\mathcal{B},\\
    \beta_0[i],&\text{ otherwise}.
    \end{cases}
\end{align*}
Note that this definition is consistent with the definition of $\beta_0$ when $\lambda=0$.
Thus, for any $\lambda\in\mathds{R}$, we have
\begin{align}
    \Xtr\beta_\lambda&=\Xtr\beta_0+\lambda\sum_{j=1}^k c_j\X_{b_j}\nonumber\\
    &=\Xtr\beta_0\text{ (by Eq.~\eqref{eq.temp_101101})}\nonumber\\
    &=\Ytr\text{ (since $\beta_0$ satisfies the constraint of Eq.~\eqref{eq.BP_origin})}.\label{eq.temp_101102}
\end{align}
In other words, any $\beta_\lambda$ also satisfies the constraint of Eq.~\eqref{eq.BP_origin}.
Define
\begin{align*}
    &\mathcal{L}\defeq\left\{i\in\mathcal{B}\ \Bigg|\ -\frac{\beta_0[i]}{c_i}<0\right\},\quad \mathcal{U}\defeq\left\{i\in\mathcal{B}\ \Bigg|\ -\frac{\beta_0[i]}{c_i}>0\right\},\\
    &\mathsf{LB}\defeq\begin{cases}
    \max_{i\in\mathcal{L}}\left(-\frac{\beta_0[i]}{c_i}\right),&\text{ if }\mathcal{L}\neq \varnothing,\\
    0,&\text{ otherwise},
    \end{cases}\\
    &\mathsf{UB}\defeq\begin{cases}
    \min_{i\in\mathcal{U}}\left(-\frac{\beta_0[i]}{c_i}\right),&\text{ if }\mathcal{U}\neq \varnothing,\\
    0,&\text{ otherwise}.
    \end{cases}
\end{align*}
Base on those definitions, we immediately have the following two properties for the interval $[\mathsf{LB},\ \mathsf{UB}]$. First,
we must have $[\mathsf{LB},\ \mathsf{UB}]\neq \varnothing$. This can be proved by contradiction. Suppose on the contrary that $[\mathsf{LB},\ \mathsf{UB}]=\varnothing$. Because by definition $\mathsf{LB}\leq 0$ and $\mathsf{UB}\geq 0$, we must have $\mathsf{LB}=\mathsf{UB}=0$. Because $\mathsf{LB}=0$, we must have $\mathcal{L}=\varnothing$. Because $\mathsf{UB}=0$, we must have $\mathcal{U}=\varnothing$. Thus, we have $\mathcal{B}=\mathcal{L}\cup \mathcal{U}=\varnothing$, which contradicts the fact that $\mathcal{B}$ is not empty. We can thus conclude that $[\mathsf{LB},\ \mathsf{UB}]\neq \varnothing$. Second, for any $\lambda\in (\mathsf{LB},\ \mathsf{UB})$, $\mathsf{sign}(\beta_0[i]+\lambda c_i)=\mathsf{sign}(\beta_0[i])$ for all $i\in\mathcal{B}$. This is because \begin{align*}
    \frac{\beta_0[i]+\lambda c_i}{\beta_0[i]}=1-\lambda\left(-\frac{c_i}{\beta_0[i]}\right)> \begin{cases}
    1-\mathsf{LB}\cdot\left(-\frac{c_i}{\beta_0[i]}\right)\geq 0,\text{ if }i\in\mathcal{L},\\
    1-\mathsf{UB}\cdot\left(-\frac{c_i}{\beta_0[i]}\right)\geq 0,\text{ if }i\in\mathcal{U}.
    \end{cases}
\end{align*}

By the second property, we can show that $\|\beta_\lambda\|_1$ is a linear function with respect to $\lambda$ when $\lambda\in[\mathsf{LB},\ \mathsf{UB}]$. Indeed, we can check that $\|\beta_\lambda\|_1$ is continuous with respect to $\lambda$ everywhere and its derivative is a constant in $\lambda\in(\mathsf{LB},\ \mathsf{UB})$, i.e.,
\begin{align}\label{eq.temp_101103}
    \frac{\partial \|\beta_\lambda\|_1}{\partial\lambda}\Bigg|_{\lambda\in(\mathsf{LB},\ \mathsf{UB})}=\sum_{i\in\mathcal{B}}c_i\cdot\mathsf{sign}(\beta_0[i]+\lambda c_i)=\sum_{i\in\mathcal{B}}c_i\cdot\mathsf{sign}(\beta_0[i]).
\end{align}

By the first property, there are only three possible cases to consider.

Case 1: $\mathsf{LB}<0$ and $\mathsf{UB}>0$. By linearity, we have
\begin{align*}
    \min\{\|\beta_{\mathsf{LB}}\|_1, \|\beta_{\mathsf{UB}}\|_1\}\leq \|\beta_0\|_1.
\end{align*}
Thus, by Eq.~\eqref{eq.temp_101102}, we know that either $\beta_{\mathsf{LB}}$ or $\beta_{\mathsf{UB}}$ (or both of them) is a solution of Eq.~\eqref{eq.BP_origin}. By the definitions of $\beta_\lambda$, $\mathsf{LB}$, and $\mathsf{UB}$, we know that both $\beta_{\mathsf{LB}}$ and $\beta_{\mathsf{UB}}$ have a strictly smaller number of non-zero elements than that of $\beta_0$ when $\mathsf{LB}\neq 0$ and $\mathsf{UB}\neq 0$. This contradicts the assumption that $\beta_0$ has the smallest number of non-zero elements.

Case 2: $\mathsf{LB}<0$ and $\mathsf{UB}=0$. Since $\mathsf{UB}=0$, we have $\mathcal{U}=\varnothing$, which implies that $\beta_0[i]/c_i>0$ for all $i\in \mathcal{B}$, i.e., $\beta_0[i]$ and $c_i$ have the same sign for all $i\in\mathcal{B}$. Thus, the value of Eq.~\eqref{eq.temp_101103} is positive, i.e., $\|\beta_\lambda\|_1$ is monotone increasing with respect to $\lambda\in[\mathsf{LB},\ \mathsf{UB}]$. Thus, we have $\|\beta_{\mathsf{LB}}\|_1\leq \|\beta_0\|_1$. By Eq.~\eqref{eq.temp_101102}, we know that $\beta_{\mathsf{LB}}$ is a solution of Eq.~\eqref{eq.BP_origin}. By the definitions of $\beta_\lambda$ and $\mathsf{LB}$, we know that $\beta_{\mathsf{LB}}$ has a strictly smaller number of non-zero elements than that of $\beta_0$ when $\mathsf{LB}\neq 0$. This contradicts the assumption that $\beta_0$ has the smallest number of non-zero elements.

Case 3: $\mathsf{LB}=0$ and $\mathsf{UB}>0$. Similar to Case 2, we can show that $\beta_{\mathsf{UB}}$ is a solution of Eq.~\eqref{eq.BP_origin} and has a strictly smaller number of non-zero elements than that of $\beta_0$. This contradicts the assumption that $\beta_0$ has the smallest number of non-zero elements.

In conclusion, all cases lead to a contradiction. The result of this lemma thus follows.
\end{proof}
\section{An estimate of \texorpdfstring{$\|\etr\|_2$ (close to $\sigma$
with high probability)}{l2 norm of epsilon(train)}}\label{ap.sec.noise}
\begin{lemma}[stated on pp. 1325 of \citep{laurent2000adaptive}]\label{le.chi_bound}
Let $U$ follow a chi-square distribution with D
degrees of freedom. For any positive x, we have
\begin{align*}
    &\prob{U-D\geq 2\sqrt{Dx}+2x}\leq e^{-x},\\
    &\prob{D-U\geq 2\sqrt{Dx}}\leq e^{-x}.
\end{align*}
\end{lemma}

Notice that $n\|\etr\|_2^2/\sigma^2$ follows the chi-square distribution with $n$ degrees of freedom. We thus have
\begin{align*}
    \prob{\|\etr\|_2^2\leq 2\sigma^2}&=1-\prob{\frac{n\|\etr\|_2^2}{\sigma^2}\geq 2n}\\
    &=1-\prob{\frac{n\|\etr\|_2^2}{\sigma^2}-n\geq n}.
\end{align*}
Now we use the fact that
\begin{align*}
    2\sqrt{n\frac{2-\sqrt{3}}{2}n}+2\cdot\frac{2-\sqrt{3}}{2}n&=\sqrt{n^2(4-2\sqrt{3})}+(2-\sqrt{3})n\\
    &=\sqrt{n^2(\sqrt{3}-1)^2}+(2-\sqrt{3})n\\
    &=(\sqrt{3}-1)n+(2-\sqrt{3})n\\
    &=n.
\end{align*}
We thus have
\begin{align}
    \prob{\|\etr\|_2^2\leq 2\sigma^2}&=1-\prob{\frac{n\|\etr\|_2^2}{\sigma^2}-n\geq 2\sqrt{n\frac{2-\sqrt{3}}{2}n}+2\cdot\frac{2-\sqrt{3}}{2}n}\nonumber\\
    &\geq 1-\exp\left(-\frac{2-\sqrt{3}}{2}n\right)\text{ (by Lemma \ref{le.chi_bound} using $x=\frac{2-\sqrt{3}}{2}n$)}.\label{eq.temp_pointer2}
\end{align}
We also have
\begin{align}
    \prob{\|\etr\|_2^2\geq \frac{\sigma^2}{2}}&=1-\prob{\frac{n\|\etr\|_2^2}{\sigma^2}\leq \frac{n}{2}}\nonumber\\
    &=1-\prob{n-\frac{n\|\etr\|_2^2}{\sigma^2}\geq \frac{n}{2}}\nonumber\\
    &=1-\prob{n-\frac{n\|\etr\|_2^2}{\sigma^2}\geq 2\sqrt{n\frac{n}{16}}}\nonumber\\
    &\geq 1-\exp\left(-\frac{n}{16}\right)\text{ (by Lemma \ref{le.chi_bound} using $x=n/16$)}.\label{eq.temp_pointer1}
\end{align}
In other words, when $n$ is large, $\|\etr\|_2^2$ should be close to $\sigma^2$. As a result, in the rest of the paper, we will use $\|\etr\|_2^2$ as a surrogate for the noise level.

\section{Proof of Lemma~\ref{lemma:distortion} (distortion of \texorpdfstring{$\underbeta$}{beta} due to normalization of \texorpdfstring{$\Xtr$}{Xtrain} is
small)}
\label{app.normalize}
From Eq.~\eqref{eq.scale_beta}, it is easy to see that the amount of distortion of $\underbeta$ depends on the size of $\mathbf{H}_i$ for those $i$ such that either $\ith{\underbeta}$ or $\ith{\underbetaBP}$ is non-zero. More precisely, we define the sets
\begin{align*}
    &\mathcal{A}\defeq\{i:\ \ith{\underbeta}\neq 0\}\cup \{i:\ \ith{\underbetaBP}\neq 0\}=\{1,2,\cdots,s\}\cup \{i:\ \ith{\underbetaBP}\neq 0\},\\
    &\mathcal{B}\defeq\mathcal{A}\setminus \{1,\cdots,s\}.
\end{align*}
Notice that because $\|\underbetaBP\|_0=\|\betaBP\|_0\leq n$, the number of elements in $\mathcal{A}$ satisfies $|\mathcal{A}|\leq s+n$. Thus, the number of elements in $\mathcal{B}$ satisfies
\begin{align}\label{eq.num_A}
    |\mathcal{B}|=|\mathcal{A}\setminus\{1,\cdots,s\}|=|\mathcal{A}|-s\leq s+n-s=n.
\end{align}
Then, we have
\begin{align}
    \|\underwBP\|_2^2=\|\underbetaBP-\underbeta\|_2^2&=\sum_{i=1}^p\frac{n(\ith{\betaBP}-\ith{\beta})^2}{\|\mathbf{H}_i\|_2^2}\nonumber\\
    &=\sum_{i\in \mathcal{A}}\frac{n(\ith{\betaBP}-\ith{\beta}^2}{\|\mathbf{H}_i\|_2^2}\nonumber\\
    &\leq\frac{n}{\min_{i\in\mathcal{A}}\|\mathbf{H}_i\|_2^2}\sum_{i\in \mathcal{A}}(\ith{\betaBP}-\ith{\beta})^2\nonumber\\
    &=\frac{n}{\min_{i\in\mathcal{A}}\|\mathbf{H}_i\|_2^2}\|\betaBP-\beta\|_2^2\nonumber\\
    &=\frac{n}{\min_{i\in\mathcal{A}}\|\mathbf{H}_i\|_2^2}\|\wBP\|_2^2.\label{eq.temp_052702}
\end{align}
In the same way, we can get the other side of the bound:
\begin{align}
    \|\underwBP\|_2^2\geq \frac{n}{\max_{i\in\mathcal{A}}\|\mathbf{H}_i\|_2^2}\|\wBP\|_2^2.\label{eq.temp_060102}
\end{align}
Similarly, for $\ell_1$-norm, we have
\begin{align}
    \|\underwBP\|_1=\|\underbetaBP-\underbeta\|_1&=\sum_{i=1}^p\frac{\sqrt{n}\left|\ith{\betaBP}-\ith{\beta}\right|}{\|\mathbf{H}_i\|_2}\nonumber\\
    &=\sum_{i\in \mathcal{A}}\frac{\sqrt{n}\left|\ith{\betaBP}-\ith{\beta}\right|}{\|\mathbf{H}_i\|_2}\nonumber\\
    &\leq \frac{\sqrt{n}}{\min_{i\in\mathcal{A}}\|\mathbf{H}_i\|_2}\sum_{i\in \mathcal{A}}|\ith{\betaBP}-\ith{\beta}|\nonumber\\
    &=\frac{\sqrt{n}}{\min_{i\in\mathcal{A}}\|\mathbf{H}_i\|_2}\|\betaBP-\beta\|_1\nonumber\\
    &=\frac{\sqrt{n}}{\min_{i\in\mathcal{A}}\|\mathbf{H}_i\|_2}\|\wBP\|_1,\label{eq.temp_052901}
\end{align}
as well as
\begin{align}
    \|\underwBP\|_1\geq \frac{\sqrt{n}}{\max_{i\in\mathcal{A}}\|\mathbf{H}_i\|_2}\|\wBP\|_1.\label{eq.temp_060103}
\end{align}
It only remains to bound the minimum or maximum of $\|\mathbf{H}_i\|_2^2$ over $i\in \mathcal{A}$. Intuitively, for each $i$, since $\mathds{E}[\|\mathbf{H}_i\|_2^2]=n$, $\|\mathbf{H}_i\|_2^2$ should be close to $n$ when $n$ is large. However, here the difficulty is that we do not know which elements $i$ belong to $\mathcal{A}$. If we were to account for all possible $i=1,2,\cdots,p$, when $p$ is exponentially large in $n$, our bounds for the minimum and maximum of $\|\mathbf{H}_i\|_2$ would become very loose. Fortunately, for those $i=s+1,\cdots,p$ (i.e., outside of the true basis), we can show that $\|\mathbf{H}_i\|_2^2$ is independent of $\mathcal{A}$. Using this fact, we can obtain a much tighter bound on the minimum and maximum of $\|\mathbf{H}_i\|_2^2$ on $\mathcal{A}$. Towards this end, we first show the following lemma:
\begin{lemma}\label{le.independent_scale}
$\betaBP$ is independent of the size  $\|\mathbf{H}_i\|_2$ of $\mathbf{H}_i$ for  $i\in \{s+1,\cdots,p\}$. In other words, scaling any $\mathbf{H}_i$ by a non-zero value $\alpha_i$ for any $i\in \{s+1,\cdots,p\}$ does not affect $\betaBP$.
\end{lemma}
\begin{proof}
Suppose that $\mathbf{H}_i$ is scaled by any $\alpha_i\neq  0$ for any $i\in \{s+1,\cdots,p\}$. We denote the new $\mathbf{H}$ matrix by $\mathbf{H}'$, i.e., $\mathbf{H}'_i=\alpha_i\mathbf{H}_i$ for some $i\in \{s+1,\cdots,p\}$. By the normalization in Eq.~(\ref{def.norm_X}), we know that $\Xtr$ does not change after this scaling. Further, because $\ith{\underbeta}=0$ for $i\in\{s+1,\cdots,p\}$, $\Ytr$ is also unchanged. Therefore, the BP solution as defined in Eq.~\eqref{eq.BP_origin} will remain the same.
\end{proof}
Let $\mathfrak{A}\subseteq\{1,\cdots,p\}$ denote any possible realization of the set $\mathcal{A}$. By Lemma~\ref{le.independent_scale} and noting that all $\mathbf{H}_i$'s are \emph{i.i.d.}, we then get that, for any $\hi\in \mathds{R}$, $i=1,\cdots,p$, and any fixed set $\mathcal{C}\subseteq \{s+1,\cdots,p\}$,
\begin{align}
    &\probcond{\mathcal{A}=\mathfrak{A},\|\mathbf{H}_i\|_2^2\geq \hi,i=1,\cdots,s}{ \|\mathbf{H}_i\|_2^2\geq \hi,\forall i\in \mathcal{C}}\nonumber\\
    &=\prob{\mathcal{A}=\mathfrak{A},\|\mathbf{H}_i\|_2^2\geq \hi,i=1,\cdots,s}.\label{eq.prob_indepent}
\end{align}
In other words, $\mathcal{A}$ and $\|\mathbf{H}_i\|_2^2$, $i=1,\cdots,s$ are independent of $\|\mathbf{H}_i\|_2^2$, $i=s+1,\cdots,p$.
Of course, this is equivalent to stating that $\|\mathbf{H}_i\|_2^2$, $i=s+1,\cdots, p$ are independent of $\mathcal{A}$ and $\|\mathbf{H}_i\|_2^2$, $i=1,\cdots,s$.
More precisely, for any $\hi\in \mathds{R}$, $i=1,\cdots,p$, and any fixed set $\mathcal{C}\subseteq \{s+1,\cdots,p\}$, we have
\begin{align}
    &\probcond{\|\mathbf{H}_i\|_2^2\geq \hi,\forall i\in \mathcal{C}}{ \mathcal{A}=\mathfrak{A},\|\mathbf{H}_i\|_2^2\geq \hi,i=1,\cdots,s}\nonumber\\
    =&\frac{\probcond{\mathcal{A}=\mathfrak{A},\|\mathbf{H}_i\|_2^2\geq \hi,i=1,\cdots,s}{ \|\mathbf{H}_i\|_2^2\geq \hi,\forall i\in \mathcal{C}}}{\prob{\mathcal{A}=\mathfrak{A},\|\mathbf{H}_i\|_2^2,i=1,\cdots,s}}\nonumber\nonumber\\
    &\cdot \Pr\left(\left\{\|\mathbf{H}_i\|_2^2\geq \hi,\forall i\in \mathcal{C}\right\}\right)\text{ (by Bayes' Theorem)}\nonumber \\
    =&\Pr\left(\left\{\|\mathbf{H}_i\|_2^2\geq \hi,\forall i\in \mathcal{C}\right\}\right)\text{ (using Eq.~\eqref{eq.prob_indepent})}.\label{eq.prob_indepent_2}
\end{align}
Further, because all $\mathbf{H}_i$'s are \emph{i.i.d.}, we have
\begin{align*}
    \Pr\left(\left\{\|\mathbf{H}_i\|_2^2\geq \hi,\forall i\in \mathcal{C}\right\}\right)=\prod_{i\in\mathcal{C}}\prob{\|\mathbf{H}_i\|_2^2\geq \hi}=\prod_{i\in\mathcal{C}}\prob{\|\mathbf{H}_1\|_2^2\geq \hi}.
\end{align*}
Substituting back to Eq.~\eqref{eq.prob_indepent_2}, we have
\begin{align}
    &\probcond{\|\mathbf{H}_i\|_2^2\geq \hi,\forall i\in \mathcal{C}}{ \mathcal{A}=\mathfrak{A},\|\mathbf{H}_i\|_2^2\geq \hi,i=1,\cdots,s}\nonumber\\
    =&\prod_{i\in\mathcal{C}}\prob{\|\mathbf{H}_1\|_2^2\geq \hi}.\label{eq.temp_060901}
\end{align}
We are now ready to bound the probability distribution of $\min_{i\in\mathcal{A}}\|\mathbf{H}_i\|_2^2$ in Eq.~\eqref{eq.temp_052702}. Because $\{1,\cdots,s\}\subseteq \mathcal{A}$, we have (recalling that $\mathcal{B}=\mathcal{A}\setminus\{1,\cdots,s\}$)
\begin{align}
    &\Pr\left(\left\{\min_{i\in\mathcal{A}}\|\mathbf{H}_i\|_2^2\geq \frac{n}{2}\right\}\right)\nonumber\\
    =&\Pr\left(\mycap_{i\in\mathcal{A}}\left\{\|\mathbf{H}_i\|_2^2\geq \frac{n}{2}\right\}\right)\nonumber\\
    =&\prob{\|\mathbf{H}_i\|_2^2\geq \frac{n}{2},i=1,\cdots,s}\cdot\probandcond{\|\mathbf{H}_i\|_2^2\geq \frac{n}{2}}{ \|\mathbf{H}_i\|_2^2\geq \frac{n}{2},i=1,\cdots,s}{\mycap_{i\in \mathcal{B}}}\nonumber\\
    =&\left(1-\prob{\|\mathbf{H}_1\|_2^2\geq \frac{n}{2}}\right)^s\cdot\probandcond{\|\mathbf{H}_i\|_2^2\geq \frac{n}{2}}{ \|\mathbf{H}_i\|_2^2\geq \frac{n}{2},i=1,\cdots,s}{\mycap_{i\in \mathcal{B}}}\label{eq.temp_060501}\\
    &\text{ (because all $\mathbf{H}_i$'s are \emph{i.i.d.})}\nonumber.
\end{align}

We first study the second term of the right-hand-side of Eq.~\eqref{eq.temp_060501} by conditioning on $\mathcal{A}=\mathfrak{A}$. For any possible realization $\mathfrak{A}$ of the set $\mathcal{A}$, we have
\begin{align}
    &\probandcond{\|\mathbf{H}_i\|_2^2\geq \frac{n}{2}}{ \mathcal{A}=\mathfrak{A},\|\mathbf{H}_i\|_2^2\geq \frac{n}{2},i=1,\cdots,s}{\mycap_{i\in \mathcal{B}}}\nonumber\\
    =&\probandcond{\|\mathbf{H}_i\|_2^2\geq \frac{n}{2}}{\mathcal{A}=\mathfrak{A}, \|\mathbf{H}_i\|_2^2\geq \frac{n}{2},i=1,\cdots,s}{\mycap_{i\in \mathfrak{A}\setminus\{1,\cdots,s\}}}\nonumber\\
    =&\prod_{i\in \mathfrak{A}\setminus\{1,\cdots,s\}}\prob{\|\mathbf{H}_1\|_2^2\geq \frac{n}{2}}\text{ (by letting $\mathcal{C}=\mathfrak{A}\setminus\{1,\cdots,s\}$ in Eq.~\eqref{eq.temp_060901})}\nonumber\\
    \geq & \left(1-\prob{\|\mathbf{H}_1\|_2^2\leq \frac{n}{2}}\right)^n\text{ (by Eq.~\eqref{eq.num_A})}.\label{eq.temp_060502}
\end{align}
Since the right-hand-side of Eq.~\eqref{eq.temp_060502} is independent of $\mathfrak{A}$, we then conclude that
\begin{align*}
    \probandcond{\|\mathbf{H}_i\|_2^2\geq \frac{n}{2}}{ \|\mathbf{H}_i\|_2^2\geq \frac{n}{2},i=1,\cdots,s}{\mycap_{i\in \mathcal{B}}}\geq  \left(1-\prob{\|\mathbf{H}_1\|_2^2\leq \frac{n}{2}}\right)^n.
\end{align*}
Substituting back to Eq.~\eqref{eq.temp_060501}, we have
\begin{align}
    \prob{\min_{i\in\mathcal{A}}\|\mathbf{H}_i\|_2^2\geq \frac{n}{2}}\geq& \left(1-\prob{\|\mathbf{H}_1\|_2^2\leq \frac{n}{2}}\right)^{n+s}\nonumber\\
    \geq &\left(1-\prob{\|\mathbf{H}_1\|_2^2\leq \frac{n}{2}}\right)^{2n}\text{ (assuming $s\leq n$)}\nonumber\\
    \geq& (1-e^{-n/16})^{2n}\label{eq.temp_052703}\\
    \geq& 1-2n\cdot e^{-n/16}\nonumber\\
    =&1-e^{-n/16+\ln (2n)},\label{eq.temp_052704}
\end{align}
where in Eq.~(\ref{eq.temp_052703}), we have used results for large deviation analysis on the probability of chi-square distribution (similar to the analysis of getting Eq.~\eqref{eq.temp_pointer1} in Appendix~\ref{ap.sec.noise}).
Using similar ideas, we can also get
\begin{align}
    \Pr\left(\left\{\max_{i\in \mathcal{A}}\|\mathbf{H}_i\|_2^2\leq 2n\right\}\right)&\geq \left(1-\Pr\left(\left\{\|\mathbf{H}_1\|_2^2\geq 2n\right\}\right)\right)^{2n} \nonumber\\
    &\geq \left(1-\exp\left(-\frac{2-\sqrt{3}}{2}n\right)\right)^{2n}\text{ (similar to Eq.~\eqref{eq.temp_pointer2} in Appendix~\ref{ap.sec.noise})}\nonumber\\
    &\geq 1-2n\cdot \exp\left(-\frac{2-\sqrt{3}}{2}n\right)\nonumber\\
    &=1-\exp\left(-\frac{2-\sqrt{3}}{2}n+\ln (2n)\right).\label{eq.temp_060101}
\end{align}

Applying Eq.~(\ref{eq.temp_052704}) in Eq.~(\ref{eq.temp_052702}) and applying Eq.~(\ref{eq.temp_060101}) in Eq.~(\ref{eq.temp_060102}), we conclude that 
\begin{align*}
    \prob{\|\underwBP\|_2\leq \sqrt{2}\|\wBP\|_2}&=\prob{\|\underwBP\|_2^2\leq 2\|\wBP\|_2^2}\\
    &\geq 1-\exp\left(-\frac{n}{16}+\ln (2n)\right),\\
    \prob{\|\wBP\|_2\leq \sqrt{2}\|\underwBP\|_2}&=\prob{\|\underwBP\|_2^2\leq 2\|\wBP\|_2^2}\\
    &\geq 1-\exp\left(-\frac{2-\sqrt{3}}{2}n+\ln (2n)\right).
\end{align*}
Applying Eq.~(\ref{eq.temp_052704}) in Eq.~(\ref{eq.temp_052901}) and applying Eq.~(\ref{eq.temp_060102}) in Eq.~(\ref{eq.temp_060103}), we conclude that 
\begin{align*}
    \prob{\|\underwBP\|_1\leq \sqrt{2}\|\wBP\|_1}&\geq 1-\exp\left(-\frac{n}{16}+\ln (2n)\right),\\
    \prob{\|\wBP\|_1\leq \sqrt{2}\|\underwBP\|_1}&\geq 1-\exp\left(-\frac{2-\sqrt{3}}{2}n+\ln (2n)\right).
\end{align*}
The result of Lemma~\ref{lemma:distortion} thus follows.

\section{Proof of Proposition~\ref{prop.bound_WB1_KM} (relationship
between \texorpdfstring{$\|\wBP\|_1$}{||wBP||1} and \texorpdfstring{$\|w^I\|_1$}{||wI||1})}
\label{app:proof:wb1}

\begin{proof}
Since we focus on $\wBP$, we rewrite BP in the form of $\wBP$. Notice that
\begin{align*}
    \|\betaBP\|_1=\|\wBP+\beta\|_1=\|\wBP_0+\beta_0\|_1+\|\wBP_1\|_1.
\end{align*}
Thus, we have
\begin{align}\label{eq.BP_w}
    &\wBP=\argmin_w \|w_0+\beta_0\|_1+\|w_1\|_1\nonumber \\
    &\st \Xtr w=\etr.
\end{align}
Define $\mathbf{G}\defeq \Xtr^T\Xtr$ and let $\mathbf{I}$ be the $p\times p$ identity matrix. Let $\abs{\cdot}$ denote the operation that takes the  component-wise absolute value of every element of a matrix. We have
\begin{align}\label{eq.lemma_begin}
    \|\etr\|_2^2&=\|\Xtr \wBP\|_2^2\nonumber\\
    &=(\wBP)^T\mathbf{G}\wBP\nonumber\\
    &=\|\wBP\|_2^2+(\wBP)^T(\mathbf{G}-\mathbf{I})\wBP\nonumber\\
    &\geq \|\wBP\|_2^2-\abs{\wBP}^T\abs{\mathbf{G}-\mathbf{I}}\abs{\wBP}\nonumber\\
    &\stackrel{(a)}{\geq} \|\wBP\|_2^2-M\abs{\wBP}^T\abs{\mathds{1}-\mathbf{I}}\abs{\wBP}\nonumber\\
    &=(1+M)\|\wBP\|_2^2-M\|\wBP\|_1^2,
\end{align}
where in step (a) $\mathds{1}$ represents a $p\times p$ matrix with all elements equal to $1$, and the step holds because $\mathbf{G}$ has diagonal elements equal to $1$ and off-diagonal elements no greater than $M$ in absolute value.
Because $w^I$ also satisfies the constraint of \eqref{eq.BP_w}, by the representation of $\wBP$ in \eqref{eq.BP_w}, we have
\begin{align*}
    \|\wBP_0+\beta_0\|_1+\|\wBP_1\|_1\leq \|w_0^I+\beta_0\|_1+\|w_1^I\|_1.
\end{align*}
By definition (\ref{eq.def_WI}), we have $w_0^I=\mathbf{0}$ and $\|w_1^I\|_1=\|w^I\|_1$. Thus, we have
\begin{align*}
    \|\wBP_0+\beta_0\|_1+\|\wBP_1\|_1\leq \|\beta_0\|_1+\|w^I\|_1.
\end{align*}
By the triangle inequality, we have $\|\beta_0\|_1-\|\wBP_0+\beta_0\|_1\leq \|\wBP_0\|_1$. Thus, we obtain
\begin{align}\label{eq.lemma_second}
    \|\wBP_1\|_1&\leq\|\beta_0\|_1-\|\wBP_0+\beta_0\|_1+\|w^I\|_1\nonumber\\
    &\leq  \|\wBP_0\|_1+\|w^I\|_1.
\end{align}
We now use \eqref{eq.lemma_begin} and \eqref{eq.lemma_second} to establish \eqref{eq.prop1}. Specifically, because $\wBP_0\in \mathds{R}^{s}$, we have
\begin{align*}
    \|\wBP_0\|_2^2\geq \frac{1}{s}\|\wBP_0\|_1^2.
\end{align*}
Thus, we have
\begin{align}\label{eq.temp_1226}
    \|\wBP\|_2^2\geq \|\wBP_0\|_2^2 \geq \frac{1}{s}\|\wBP_0\|_1^2.
\end{align}
Applying Eq. \eqref{eq.lemma_second}, we have
\begin{align}\label{eq.temp_1226_2}
    \|\wBP\|_1=\|\wBP_1\|_1+\|\wBP_0\|_1\leq 2\|\wBP_0\|_1+\|w^I\|_1.
\end{align}
Substituting Eq. \eqref{eq.temp_1226} and Eq. \eqref{eq.temp_1226_2} in Eq. \eqref{eq.lemma_begin}, we have
\begin{align*}
    \frac{1+M}{s}\|\wBP_0\|_1^2-M(2\|\wBP_0\|_1+\|w^I\|_1)^2\leq \|\etr\|_2^2,
\end{align*}
which can be rearranged into a quadratic inequality in $\|\wBP_0\|_1$, i.e.,
\begin{align*}
    \left(\frac{1+M}{s}-4M\right)\|\wBP_0\|_1^2&-4M\|w^I\|_1\|\wBP_0\|_1\\
    &-\left(M\|w^I\|_1^2+\|\etr\|_2^2\right)\leq 0.
\end{align*}
Since $K=\frac{1+M}{sM}-4>0$, we have the leading coefficient $\frac{1+M}{s}-4M=KM>0$. Solving this quadratic inequality for $\|\wBP_0\|_1$, we have
\begin{align*}
    \|\wBP_0\|_1\leq& \frac{4M\|w^I\|_1+\sqrt{(4M\|w^I\|_1)^2+4KM\left(M\|w^I\|_1^2+\|\etr\|_2^2\right)}}{2KM}\\
    =&\frac{2\|w^I\|_1+\sqrt{4\|w^I\|_1^2+K(\|w^I\|_1^2+\frac{1}{M}\|\etr\|_2^2)}}{K}.
\end{align*}
Plugging the result into Eq. \eqref{eq.temp_1226_2}, we have
\begin{align*}
    &\|\wBP\|_1\leq\frac{4\|w^I_1\|_1+2\sqrt{4\|w^I\|_1^2+K(\|w^I\|_1^2+\frac{1}{M}\|\etr\|_2^2)}}{K}+\|w^I\|_1.
\end{align*}
This expression already provides an upper bound on $\|\wBP\|_1$ in terms of $M$ and $\|w^I\|_1$. To obtain an even simpler equation, combining $4\|w^I\|_1/K$ with $\|w^I\|_1$, and breaking the square root apart by $\sqrt{a+b+c}\leq \sqrt{a}+\sqrt{b}+\sqrt{c}$, we have 
\begin{align*}
    \|\wBP\|_1\leq& \frac{K+4}{K}\|w^I\|_1+\sqrt{\left(\frac{4\|w^I\|_1}{K}\right)^2}+\sqrt{\frac{4\|w^I\|_1^2}{K}}\\
    &+\sqrt{\frac{4\|\etr\|_2^2}{MK}}\\
    =&\left(1+\frac{8}{K}+2\sqrt{\frac{1}{K}}\right)\|w^I\|_1+\frac{2\|\etr\|_2}{\sqrt{KM}}.
\end{align*}
The result of the proposition thus follows.
\end{proof}

\section{Proof of Proposition~\ref{prop.wB2_wB1} (relationship between
\texorpdfstring{$\|\wBP\|_2$}{||wBP||2} and \texorpdfstring{$\|\wBP\|_1$}{||wBP||1})}
\label{app:proof:wb2}

\begin{proof}
In the proof of Proposition \ref{prop.bound_WB1_KM}, we have already proven Eq. \eqref{eq.lemma_begin}\footnote{Notice that in the proof of Proposition \ref{prop.bound_WB1_KM}, to get Eq. \eqref{eq.lemma_begin}, we do not need $K>0$.}. By Eq. \eqref{eq.lemma_begin}, we have
\begin{align*}
    \|\wBP\|_2\leq& \sqrt{\frac{\|\etr\|_2^2+M\|\wBP\|_1^2}{1+M}}\\
    \leq & \sqrt{\|\etr\|_2^2+M\|\wBP\|_1^2}\\
    \leq & \|\etr\|_2+\sqrt{M}\|\wBP\|_1.
\end{align*}

\end{proof}

\section{Proof of Theorem~\ref{th.main} (upper bound of model error)}
\label{app:proof:main}
The proof consists three steps. In step 1, we verify the conditions for Proposition \ref{prop.new_WI} and get the estimation on $\|w^I\|_1$ by Proposition \ref{prop.new_WI}. In step 2, we verify the conditions for Proposition \ref{prop.M} and get the estimation on $M$ by Proposition \ref{prop.M}. In step 3, we combine results in steps 1 and 2 to prove Theorem~\ref{th.main}.
\subsection*{Step 1}
We first verify that the conditions for Proposition \ref{prop.new_WI} are satisfied. Towards this end, from the assumption of Theorem \ref{th.main} that
\begin{align*}
    p\in \left[(16n)^4,\ \exp\left(\frac{n}{1792s^2}\right)\right],
\end{align*}
we have
\begin{align}
    p\geq (16n)^4,\label{eq.temp_proof_p}
\end{align}
and
\begin{align}
    p\leq \exp\left(\frac{n}{1792s^2}\right)\leq e^{n/1792}\text{ (since $s\geq 1$)}.\label{eq.temp_proof_p_u_forM}
\end{align}
Further, from the assumption of the theorem that $s\leq \sqrt{\frac{n}{7168\ln (16n)}}$, we have
\begin{align}
    n\geq s^2\cdot 7168\ln (16n)\geq 7168>100\text{ (since $s\geq 1$ and $n\geq 1$)}.\label{eq.temp_proof_n}
\end{align}
Eq. \eqref{eq.temp_proof_n} and Eq. \eqref{eq.temp_proof_p} imply that the condition of Proposition \ref{prop.new_WI} is satisfied. We thus have, from Proposition \ref{prop.new_WI}, with probability at least $1-2e^{-n/4}$,
\begin{align*}
    \|w^I\|_1 &\leq \sqrt{1+\frac{3n/2}{\ln p}}\|\etr\|_2.
\end{align*}
From Eq. \eqref{eq.temp_proof_p_u_forM}, we have
\begin{align*}
    &p\leq e^{n/1792}\leq e^{n/2}\\
    \implies & 1\leq \frac{n/2}{\ln p}.
\end{align*}
Therefore, we have
\begin{align}
    \prob{\|w^I\|_1\leq \sqrt{\frac{2n}{\ln p}}\|\etr\|_2}\geq 1-2e^{-n/4}.\label{eq.temp_proof_wI}
\end{align}
\subsection*{Step 2}
Note that Eq. \eqref{eq.temp_proof_p_u_forM} implies that the conditions of Proposition \ref{prop.M} is satisfied. We thus have, from Proposition \ref{prop.M},
\begin{align}
    \prob{M\leq 2\sqrt{7}\sqrt{\frac{\ln p}{n}}}\geq 1-2e^{-\ln p}-2e^{-n/144}.\label{eq.temp_proof_M}
\end{align}
\subsection*{Step 3}
In this step, we will combine results in steps 1 and 2 and proof the final result of Theorem \ref{th.main}. Towards this end, notice that for any event $A$ and any event $B$, we have
\begin{align*}
    \proband{A}{B}&=\prob{A}+\prob{B}-\probor{A}{B}\\
    &\geq \prob{A}+\prob{B}-1.
\end{align*}
Thus, by Eq. \eqref{eq.temp_proof_wI} and Eq. \eqref{eq.temp_proof_M}, we have
\begin{align}
    &\proband{\|w^I\|_1\leq \sqrt{\frac{2n}{\ln p}}\|\etr\|_2}{M\leq 2\sqrt{7}\sqrt{\frac{\ln p}{n}}}\label{eq.temp_event}\\
    &\geq 1-2e^{-n/4}-2e^{-\ln p}-2e^{-n/144}\nonumber\\
    &\geq 1-6e^{-\ln p}\text{ (since $\ln p\leq n/144\leq n/4$ by Eq. \eqref{eq.temp_proof_p_u_forM})}\nonumber\\
    &=1-6/p.\nonumber
\end{align}
It remains to show that the event in \eqref{eq.temp_event} implies Eq. \eqref{eq.main_bound}. Towards this end, note that from $M\leq 2\sqrt{7}\sqrt{\frac{\ln p}{n}}$, we have
\begin{align}
    K&=\frac{1+M}{sM}-4\text{ (by definition in Eq. (\ref{eq.def_K}))}\nonumber\\
    &\geq \frac{1}{sM}-4.\label{eq.temp_0127_1}
\end{align}
From the assumption of the theorem, we have
\begin{align}
    &\exp\left(\frac{n}{1792s^2}\right)\geq p\nonumber\\
    \implies & \frac{n}{1792s^2}\geq \ln p\nonumber\\
    \implies & s\leq \sqrt{\frac{n}{1792\ln p}}=\frac{1}{16\sqrt{7}}\sqrt{\frac{n}{\ln p}}.\label{eq.temp_proof_s}
\end{align}
Applying Eq. \eqref{eq.temp_proof_s} to Eq. \eqref{eq.temp_0127_1}, we have
\begin{align*}
    K&\geq \frac{1}{\frac{1}{16\sqrt{7}}\sqrt{\frac{n}{\ln p}}\cdot 2\sqrt{7}\sqrt{\frac{\ln p}{n}}}-4\\
    &=8-4=4.
\end{align*}
Applying
\begin{align}
    M\leq 2\sqrt{7}\sqrt{\frac{\ln p}{n}},\ 
    \|w^I\|_1\leq \sqrt{\frac{2n}{\ln p}}\|\etr\|_2,\text{ and }K\geq 4.\label{eq.estimate_M_wI_K}
\end{align}
to Corollary \ref{coro.wB2_K}, we have
\begin{align*}
    \|\wBP\|_2\leq& 2\|\etr\|_2+\sqrt{2\sqrt{7}}\left(\frac{\ln p}{n}\right)^{1/4}\cdot 4\cdot \sqrt{\frac{2n}{\ln p}}\|\etr\|_2\\
    =& \left(2+8\left(\frac{7n}{\ln p}\right)^{1/4}\right)\|\etr\|_2.
\end{align*}
The result of Theorem \ref{th.main} thus follows.

\section{Proof of Corollary \ref{coro.only_s} (descent floor)}
\label{app.proof_s}
\begin{proof}
For any $a\geq 1$, we have
\begin{align*}
    &\left\lfloor e^a\right\rfloor -e^{a/2}\geq e^a-e^{a/2}-1=e^{a/2}(e^{a/2}-1)-1\\
    &\geq \sqrt{e}(\sqrt{e}-1)-1=e-\sqrt{e}-1\approx 0.0696.
\end{align*}
It implies that $\left\lfloor e^a\right\rfloor\geq e^{a/2}$ for any $a\geq 1$. Taking logarithm at both sides, we have $\ln\left\lfloor e^a\right\rfloor \geq a/2$ for any $a\geq 1$. When $s\leq \sqrt{\frac{n}{7168\ln (16n)}}$, we have
\begin{align*}
    \frac{n}{1792s^2}\geq 4\ln (16n)\geq 1.
\end{align*}
Thus, by the choice of $p$ in the corollary, we have
\begin{align}\label{eq.temp0118}
    \ln p=\ln \left\lfloor\exp\left(\frac{n}{1792s^2}\right)\right\rfloor\geq \frac{n}{3584 s^2}.
\end{align}
Substituting Eq. \eqref{eq.temp0118} into Eq. \eqref{eq.main_bound}, we have
\begin{align*}
    \frac{\|\wBP\|_2}{\|\etr\|_2}&\leq 2+8\left(7\times 3584 s^2\right)^{1/4}\\
    &=2+32\sqrt{14}\sqrt{s}.
\end{align*}
\end{proof}

\section{Proof of Proposition \ref{prop.new_WI} (upper bound of 
\texorpdfstring{$\|w^I\|_1$)}{WI l1 norm}}
\label{app:proof:wI}
Recall that, by the definition of $w^I$ in Eq. (\ref{eq.def_WI}), $w^I$ is independent of the first $s$ columns of $\Xtr$. For ease of exposition, let $\mathbf{A}$ denote a $n\times (p-s)$ sub-matrix of $\Xtr$ that consists of the last $(p-s)$ columns, i.e.,
\begin{align*}
    \mathbf{A}\defeq[\X_{s+1}\ \X_{s+2}\ \cdots\ \X_{p}].
\end{align*}
Thus, $\|w^I\|_1$ equals to the optimal objective value of
\begin{align}\label{prob.original}
    \min_{\alpha\in\mathds{R}^{p-s}} \|\alpha\|_1\st \mathbf{A}\alpha=\etr.
\end{align}
Let $\lambda$ be a $n\times 1$ vector that denotes the Lagrangian multiplier associated with the constraint $\mathbf{A}\alpha=\etr$. Then, the Lagrangian of the problem (\ref{prob.original}) is
\begin{align*}
    L(\alpha, \lambda)\defeq\|\alpha\|_1+\lambda^T(\mathbf{A}\alpha-\etr).
\end{align*}
Thus, the dual problem is
\begin{align}\label{prob.dual_origin}
    \max_{\lambda} h(\lambda),
\end{align}
where the dual objective function is given by
\begin{align*}
    h(\lambda)=\inf_{\alpha}L(\alpha, \lambda).
\end{align*}
Let $\mathbf{A}_i$ denote the $i$-th column of $\mathbf{A}$. It is easy to verify that
\begin{align*}
    &h(\lambda)=\inf_{\alpha}L(\alpha, \lambda)\\
    &=
    \begin{cases}
    -\infty\quad &\text{ if there exists $i$ such that }|\lambda^T\mathbf{A}_i|>1,\\
    -\lambda^T \etr&\text{ otherwise}.
    \end{cases}
\end{align*}
Thus, the dual problem (\ref{prob.dual_origin}) is equivalent to
\begin{align}\label{prob.dual}
    &\max_{\lambda} \lambda^T (-\etr) \nonumber\\
    &\st -1\leq \lambda^T \mathbf{A}_i \leq 1\ \forall i\in\{1,2,\cdots, p-s\}.
\end{align}

This dual formulation gives the following geometric interpretation. Consider the $\mathds{R}^n$ space that $\lambda$ and $\mathbf{A}_i$ stay in. Since $\|\mathbf{A}_i\|_2=1$, the constraint $-1\leq \lambda^T\mathbf{A}_i\leq 1$ corresponds to the region between two parallel hyperplanes that are tangent to a unit hyper-sphere at $\mathbf{A}_i$ and $-\mathbf{A}_i$, respectively.
Intuitively, as $p$ goes to infinity, there will be an infinite number of such hyperplanes. Since $\mathbf{A}_i$ is uniformly random on the surface of a unit hyper-sphere, as $p$ increases, more and more such random hyperplanes ``wrap" around the hyper-sphere. Eventually, the remaining feasible region becomes a unit ball.
This implies that the maximum value of the problem (\ref{prob.dual}) becomes $\|\etr\|_2$ when $p$ goes to infinity and the optimal $\lambda$ is attained when $\lambda^*=-\etr/\|\etr\|_2$. Our result in Proposition \ref{prop.new_WI} is also consistent with this intuition that $\|w^I\|_1\rightarrow \|\etr\|_2$ as $p\rightarrow \infty$. Of course, the challenge of Proposition \ref{prop.new_WI} is to establish an upper bound of $\|w^I\|_1$ even for finite $p$, which we will study below.

Another intuition from this geometric interpretation is that, among all $\mathbf{A}_i$'s, those ``close" to the direction of $\pm\etr$ matter most, because their corresponding hyperplanes are the ones that wrap the unit hyper-sphere around the point $\lambda^*=-\etr/\|\etr\|_2$. Next, we construct an upper bound of \eqref{prob.dual} by using $q$ such ``closest" $\mathbf{A}_i$'s.

Specifically, 
for all $i\in \{1,2,\cdots, p-s\}$, we define
\begin{align*}
    \mathbf{B}_i\defeq\begin{cases}
    \mathbf{A}_i&\text{ if }\mathbf{A}_i^T(-\etr)\geq 0,\\
    -\mathbf{A}_i&\text { otherwise}.
    \end{cases}
\end{align*}
Then, we sort $\mathbf{B}_i$ according to the inner product $\mathbf{B}_i^T(-\etr)$. Let $\mathbf{B}_{(1)},\cdots,\mathbf{B}_{(q)}$ be the $q<p-s$ vectors with the largest inner products, i.e,
\begin{align}\label{eq.sorted_B}
    \mathbf{B}_{(1)}^T(-\etr)\geq \mathbf{B}_{(2)}^T(-\etr)\geq \cdots \geq \mathbf{B}_{(q)}^T(-\etr)\geq 0.
\end{align}
We then relax the dual problem (\ref{prob.dual}) to
\begin{align}\label{prob.temp1}
    &\max_{\lambda}\lambda^T (-\etr)\nonumber\\
    &\st  \lambda^T\mathbf{B}_{(i)}\leq 1\ \forall i\in\{1,2,\cdots ,q\}.
\end{align}
Note that the constraints in \eqref{prob.temp1} are a subset of those in \eqref{prob.dual}. Thus, the optimal objective value of \eqref{prob.temp1} is an upper bound on that of \eqref{prob.dual}.
\begin{figure}[ht!]
    \centering
    \includegraphics[width=3.0in]{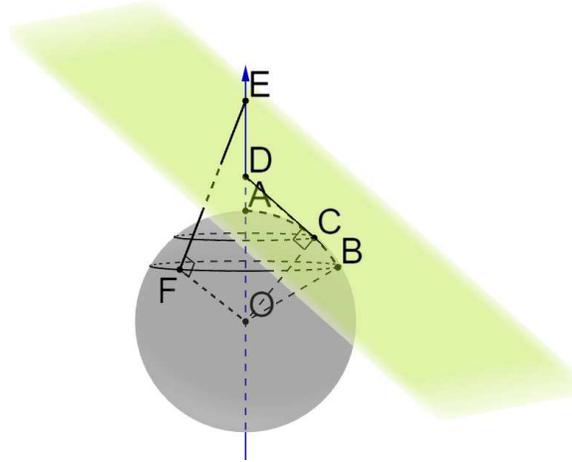}
    \caption{A 3-D geometric interpretation of Problem \eqref{prob.temp1}.}
    \label{fig.ball}
\end{figure}
\begin{figure}[ht!]
    \centering
    \includegraphics[width=3.0in]{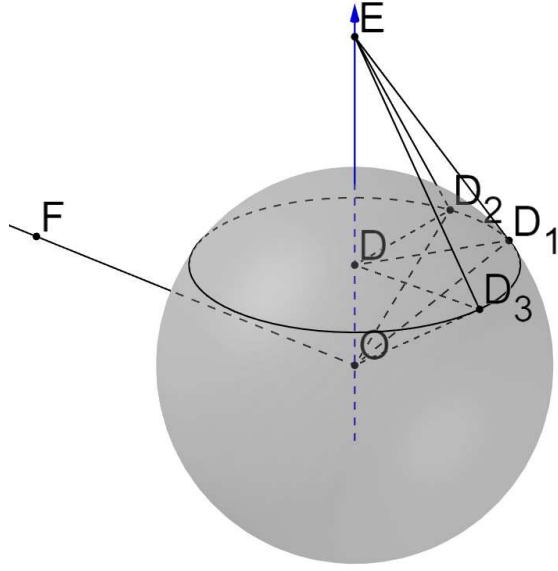}
    \caption{When all the points lie on some hemisphere, the objective value of Problem \eqref{prob.last_relax} can be infinity $\lambda$ takes the direction $\protect\overrightarrow{\text{OF}}$.}
    \label{fig.ball2}
\end{figure}

Fig. \ref{fig.ball} gives an geometric interpretation of \eqref{prob.temp1}. In Fig. \ref{fig.ball}, the gray sphere centered at the origin $O$ denotes the unit hyper-sphere in $\mathds{R}^n$. The top (north pole) of the sphere $O$ is denoted by the point $A$. The north direction denotes the direction of $(-\etr)$. The vector $\overrightarrow{\text{OC}}$ denotes some $\mathbf{B}_{(i)}$, $i\in\{1,\cdots,q-1\}$. The green plane is tangent to the sphere $O$ at the point $C$. Thus, the space below the green plane denotes the feasible region defined by the constraint $\lambda^T\mathbf{B}_{(1)}\leq 1$. The point $D$ denotes the intersection of the axis $\overrightarrow{\text{OA}}$ and the green plane. Similarly, the vector $\overrightarrow{\text{OF}}$ corresponds to $\mathbf{B}_{(q)}$. Note that its corresponding hyperplane (not drawn in Fig. \ref{fig.ball}) intersects the axis $\overrightarrow{\text{OA}}$ at a higher point $E$. This suggests that, by replacing the vector $\mathbf{B}_{(i)}$ in each of the constraints of \eqref{prob.temp1} by another vector that has a smaller inner-product with $(-\etr)$, the optimal objective value of \eqref{prob.temp1} will be even higher. For example, in Fig. \ref{fig.ball}, the constraint corresponding to $\overrightarrow{\text{OC}}$ is replaced by that corresponding to $\overrightarrow{\text{OB}}$. This procedure is made precise below.


For each $i\in\{1,2,\cdots,q\}$, we define
\begin{align}
    \mathbf{C}_{(i)}\defeq & \frac{\sqrt{1-\left(\frac{\mathbf{B}_{(q)}^T(-\etr)}{\|\etr\|_2}\right)^2}}{\sqrt{1-\left(\frac{\mathbf{B}_{(i)}^T(-\etr)}{\|\etr\|_2}\right)^2}}\cdot \left(\mathbf{B}_{(i)}-\frac{\mathbf{B}_{(i)}^T(-\etr)}{\|\etr\|_2^2}(-\etr)\right)\nonumber\\
    &+\frac{\mathbf{B}_{(q)}^T(-\etr)}{\|\etr\|_2^2}(-\etr).\label{eq.def_Ci}
\end{align}
By the definition of $\mathbf{C}_{(i)}$, it is easy to verify that $\|\mathbf{C}_{(i)}\|_2=1$ and $\mathbf{C}_{(i)}^T(-\etr)=\mathbf{B}_{(q)}^T(-\etr)\leq \mathbf{B}_{(i)}^T(-\etr)$, for all $i\in\{1,\cdots,q\}$. Roughly speaking, $\mathbf{C}_{(i)}$ is the point on the unit-hyper-sphere that is along the same (vertical) longitude as $\mathbf{B}_{(i)}$, but at the same (horizontal) latitude as $\mathbf{B}_{(q)}$.



Then, we can construct another problem as follows:
\begin{align}\label{prob.last_relax}
    &\max_{\lambda}\lambda^T(-\etr)\st\nonumber\\  &\lambda^T\mathbf{C}_{(i)}\leq 1,\ \forall i\in\{1,2,\cdots, q\}.
\end{align}
The following lemma shows that the solution to \eqref{prob.last_relax} is an upper bound on that of \eqref{prob.temp1}.
\begin{lemma}\label{le.from_B_to_C}
The objective value of Problem (\ref{prob.last_relax}) must be greater than or equal to that of Problem (\ref{prob.temp1}).
\end{lemma}
See Appendix \ref{app.proof_B_to_C} for the proof. We draw the geometric interpretation of the problem (\ref{prob.last_relax}) in Fig. \ref{fig.ball2}. Vectors $\overrightarrow{\text{OD}_1}$, $\overrightarrow{\text{OD}_2}$, and $\overrightarrow{\text{OD}_3}$ represent those vectors $\mathbf{C}_{(i)}$. Since all $\mathbf{C}_{(i)}$'s have the same latitude, points $D_1$, $D_2$, and $D_3$ locate on one circle centered at point $D$ (the circle is actually a hyper-sphere in $\mathds{R}^{n-1}$). Therefore, tangent planes on those points have the same intersection point $E$ with the axis $\overrightarrow{\text{OD}}$.

We wish to argue that the vector $\overrightarrow{\text{OE}}$ is the optimal $\lambda$ for the problem (\ref{prob.last_relax}). However, it is not always the case. Specifically, when all those $\mathbf{C}_{(i)}$'s lie on some hemisphere in $\mathds{R}^{n-1}$, we can find a direction $\lambda$ such that $\lambda^T(-\etr)$ goes to infinity. For example, in Fig. \ref{fig.ball2}, the direction $\overrightarrow{\text{OF}}$ corresponds to such a direction of $\lambda$ that $\lambda^T(-\etr)$ goes to infinity. Fortunately, when $q$ is large enough, the probability that all $\mathbf{C}_{(i)}$'s lie on some hemisphere in $\mathds{R}^{n-1}$ is very small. Towards this end, we can utilize the following result from  \citep{wendel1962problem}.
\begin{lemma}[From \citep{wendel1962problem}]\label{le.hemisphere}
Let N points be scattered uniformly at random on the surface of a sphere in an $n$-dimensional space. Then, the probability that all the points lie on some hemisphere equals to
\begin{align*}
    2^{-N+1}\sum_{k=0}^{n-1}\binom{N-1}{k}.
\end{align*}
\end{lemma}
Applying Lemma \ref{le.hemisphere} to all $q$ points $ \mathbf{C}_{(1)},\cdots,\mathbf{C}_{(q)}$ (represented by $D_1,D_2,D_3$ in Fig. \ref{fig.ball2}) on the sphere in $\mathds{R}^{n-1}$, we can quantify the probability that the situation in Fig. \ref{fig.ball2} does not happen, in which case we can then prove that the vector $\overrightarrow{\text{OE}}$ is the optimal $\lambda$ for the problem (\ref{prob.last_relax}). Lemma \ref{le.that_is_q_th} below summarizes this result.
\begin{lemma}\label{le.that_is_q_th}
The problem (\ref{prob.last_relax}) achieves the optimal objective value at
\begin{align*}
    \lambda_*=\frac{-\etr}{\mathbf{B}_{(q)}^T(-\etr)}
\end{align*}
with the probability at least
\begin{align*}
    1-2^{-q+1}\sum_{i=0}^{n-2}\binom{q-1}{i}\geq 1-e^{-(q/4-n)}.
\end{align*}
\end{lemma}
See Appendix \ref{app.proof_lemma_10} for the proof.
Letting $q=5n$, and combining Lemmas \ref{le.from_B_to_C} and \ref{le.that_is_q_th}, we have the following corollary.
\begin{corollary}\label{coro.bound_by_Aq}
The following holds
\begin{align*}
    \|w^I\|_1\leq \frac{\|\etr\|_2^2}{\mathbf{B}_{(5n)}^T(-\etr)}
\end{align*}
with probability at least
$1-e^{-n/4}$.
\end{corollary}
It only remains to bound $\mathbf{B}_{(i)}(-\etr)$. Using the fact that each $\mathbf{B}_i$ is \emph{i.i.d.} and uniformly distributed on the unit-hyper-hemisphere in $\mathds{R}^n$, we have the following result.
\begin{lemma}\label{le.estimate_Aq}
When $n\geq 100$ and $p\geq (16n)^4$, the following holds
\begin{align*}
    \mathbf{B}_{(5n)}(-\etr)\geq \frac{\|\etr\|_2}{\sqrt{1+\frac{3n/2}{\ln p}}}
\end{align*}
with probability at least $1-e^{-5n/4}$.
\end{lemma}
See Appendix \ref{app.proof_of_lemma_12} for the proof. Combining Corollary \ref{coro.bound_by_Aq} and Lemma \ref{le.estimate_Aq}, we then obtain Proposition \ref{prop.new_WI}.

\section{Proofs of supporting results in Appendix~\ref{app:proof:wI}}

\subsection{Proof of Lemma \ref{le.from_B_to_C}}\label{app.proof_B_to_C}
The proof consists of two steps. In step 1, we will define an intermediate problem \eqref{prob.with_greater_than_B1} below, and show that problem \eqref{prob.temp1} is equivalent to the problem \eqref{prob.with_greater_than_B1}. In step 2, we will show that the any feasible $\lambda$ for the problem \eqref{prob.with_greater_than_B1} is also feasible for the problem \eqref{prob.last_relax}. The conclusion of Lemma \ref{le.from_B_to_C} thus follows.

For step 1, the intermediate problem is defined as follows.
\begin{align}\label{prob.with_greater_than_B1}
    &\max_{\lambda}\lambda^T (-\etr)\st\nonumber \\
    &\lambda^T(-\etr)\geq \mathbf{B}_{(1)}^T(-\etr),\nonumber\\ &\lambda^T\mathbf{B}_{(i)}\leq 1\ \forall i\in\{1,2,\cdots ,q\}.
\end{align} 
In order to show that this problem is equivalent to \eqref{prob.temp1}, we use the following lemma.
\begin{lemma}\label{le.lower_bound_wI_B}
The value of the problem (\ref{prob.temp1}) is at least $\mathbf{B}_{(1)}^T(-\etr)$.
\end{lemma}
\begin{proof}
Because $\left|\mathbf{B}_{(1)}^T\mathbf{A}_i\right|\leq \|\mathbf{B}_{(1)}\|_2\|\mathbf{B}_{(i)}\|_2=1$ for all $i\in\{1,\cdots,q\}$, $\mathbf{B}_{(1)}$ is feasible for the problem (\ref{prob.temp1}). The result of this lemma thus follows.
\end{proof}
By this lemma, we can add an additional constraint $\lambda^T(-\etr)\geq \mathbf{B}_{(1)}^T(-\etr)$ to the problem (\ref{prob.temp1}) without affecting its solution. This is exactly problem \eqref{prob.with_greater_than_B1}. Thus, the problem (\ref{prob.temp1}) is equivalent to the intermediate problem \eqref{prob.with_greater_than_B1}, i.e., step 1 has been proven. Then, we move on to step 2. We will first use Lemma \ref{le.relax_is_relax} to show that if $\mathbf{C}_{(i)}$ can be written in the form of
\begin{align}\label{eq.temp_0128_1}
    \mathbf{C}_{(i)}=\frac{\mathbf{B}_{i}+k\etr}{\|\mathbf{B}_{(i)}+k\etr\|_2},
\end{align}
for some $k>0$ and $\mathbf{C}_{(i)}^T\etr\leq 0$
, then any $\lambda$ that satisfies $\lambda^T\mathbf{B}_{(i)}\leq 1$ and $\lambda^T(-\etr)\geq \mathbf{B}_{(1)}^T(-\etr)$ must also satisfies $\lambda^T\mathbf{C}_{(i)}\leq 1$. After that, we use Lemma \ref{le.C_has_k} to show that all $\mathbf{C}_{(i)}$'s indeed can be expressed in this form. The conclusion of step 2 then follows. Towards this end, Lemma \ref{le.relax_is_relax} is as follows.
\begin{lemma}\label{le.relax_is_relax}
For all $i\in\{1,2,\cdots,q\}$, for any $\lambda$ that satisfy
\begin{align*}
     &\lambda^T\mathbf{B}_i\leq 1,\\ &\lambda^T(-\etr) \geq \mathbf{B}_{(1)}^T(-\etr),
\end{align*}
we must have
\begin{align*}
    \lambda^T\frac{\mathbf{B}_i+k\etr}{\|\mathbf{B}_i+k\etr\|_2}\leq 1,
\end{align*}
for any $k\geq 0$ that satisfies $(\mathbf{B}_i+k\etr)^T\etr\leq 0$.
\end{lemma}
\begin{proof}
We have
\begin{align*}
    &\frac{\lambda^T\mathbf{B}_i+\lambda^T k\etr}{\|\mathbf{B}_i+k\etr\|_2}\stackrel{(i)}{\leq}\frac{\lambda^T\mathbf{B}_i+\mathbf{B}_i^T k\etr}{\|\mathbf{B}_i+k\etr\|_2}\stackrel{(ii)}{=}\frac{1+\mathbf{B}_i^T k\etr}{\|\mathbf{B}_i+k\etr\|_2}\\
    &\stackrel{(iii)}{\leq }\mathbf{B}_i^T\frac{\mathbf{B}_i+k\etr}{\|\mathbf{B}_i+k\etr\|_2}\stackrel{(iv)}{\leq} \|\mathbf{B}_i\|_2\frac{\|\mathbf{B}_i+k\etr\|_2}{\|\mathbf{B}_i+k\etr\|_2}\stackrel{(v)}{=}1.
\end{align*}
Here are reasons of each step: (i) By Eq. \eqref{eq.sorted_B}, we have $\lambda^T(-\etr)\geq\mathbf{B}_{(1)}^T(-\etr)\geq \mathbf{B}_i^T(-\etr)$. Thus, we have $\lambda^T k\etr\leq \mathbf{B}_i^T k\etr$; (ii) $\lambda^T\mathbf{B}_i\leq 1$ by the assumption of the lemma; (iii) $\mathbf{B}_i^T\mathbf{B}_i=1$ by definition of $\mathbf{B}_i$; (iv) Cauchy–Schwarz inequality; (v) $\|\mathbf{B}_i\|_2=\mathbf{B}_i^T\mathbf{B}_i=1$.


\end{proof}

Then, it only remains to prove that all $\mathbf{C}_{(i)}$'s in Eq. \eqref{eq.def_Ci} can be expressed in the specific form described above in Eq. \eqref{eq.temp_0128_1}. Towards the end, we need the following lemma, which characterizes important features of $\mathbf{C}_{(i)}$.
\begin{lemma}\label{le.C_is_one}
For any $i\in\{1,\cdots, q\}$, we must have $\|\mathbf{C}_{(i)}\|_2=1$ ,and $\mathbf{C}_{(i)}^T(-\etr)=\mathbf{B}_{(q)}(-\etr)$.
\end{lemma}
\begin{proof}
It is easy to verify that $\mathbf{C}_{(i)}^T(-\etr)=\mathbf{B}_{(q)}^T(-\etr)$. Here we show how to prove $\|\mathbf{C}_{(i)}\|_2=1$.
Because
\begin{align}\label{eq.temp_ortho}
    \left(\mathbf{B}_{(i)}-\frac{\mathbf{B}_{(i)}^T(-\etr)}{\|\etr\|_2^2}(-\etr)\right)^T(-\etr)=0,
\end{align}
we know that the first and the second term on the right hand side (RHS) of Eq. \eqref{eq.def_Ci} are orthogonal. Thus, we have
\begin{align}\label{eq.temp_decomp_C}
    \|\mathbf{C}_{(i)}\|_2^2=\|\text{1st term on the RHS of Eq. \eqref{eq.def_Ci}}\|_2^2+\|\text{2nd term on the RHS of Eq. \eqref{eq.def_Ci}}\|_2^2.
\end{align}
By Eq. \eqref{eq.temp_ortho}, we also have
\begin{align*}
    \left\|\frac{\mathbf{B}_{(i)}^T(-\etr)}{\|\etr\|_2^2}(-\etr)\right\|_2^2+\left\|\mathbf{B}_{(i)}-\frac{\mathbf{B}_{(i)}^T(-\etr)}{\|\etr\|_2^2}(-\etr)\right\|_2^2=\|\mathbf{B}_{(i)}\|_2^2=1.
\end{align*}
Notice that
\begin{align*}
    \left\|\frac{\mathbf{B}_{(i)}^T(-\etr)}{\|\etr\|_2^2}(-\etr)\right\|_2=\frac{\mathbf{B}_{(i)}^T(-\etr)}{\|\etr\|_2}.
\end{align*}
Thus, we have
\begin{align*}
    \left\|\mathbf{B}_{(i)}-\frac{\mathbf{B}_{(i)}^T(-\etr)}{\|\etr\|_2^2}(-\etr)\right\|_2=\sqrt{1-\left(\frac{\mathbf{B}_{(i)}^T(-\etr)}{\|\etr\|_2}\right)^2}.
\end{align*}
Thus, we have
\begin{align*}
    &\|\text{1st term on the RHS of Eq. \eqref{eq.def_Ci}}\|_2^2=1-\left(\frac{\mathbf{B}_{(q)}^T(-\etr)}{\|\etr\|_2}\right)^2,\\
    &\|\text{2nd term on the RHS of Eq. \eqref{eq.def_Ci}}\|_2^2=\left(\frac{\mathbf{B}_{(q)}^T(-\etr)}{\|\etr\|_2}\right)^2.
\end{align*}
Applying those to Eq. \eqref{eq.temp_decomp_C}, we then have $\|\mathbf{C}_{(i)}\|_2=1$.
\end{proof}

Finally, the following lemma shows that $\mathbf{C}_{(i)}$ can be written in the specific form in Eq. \eqref{eq.temp_0128_1}.
\begin{lemma}\label{le.C_has_k}
Each $\mathbf{C}_{(i)}$ defined in Eq. \eqref{eq.def_Ci} satisfies that $\mathbf{C}_{(i)}\etr\leq 0$ and
\begin{align}\label{eq.temp_0122}
    \mathbf{C}_{(i)}=\frac{\mathbf{B}_{(i)}+k_{(i)}\etr}{\|\mathbf{B}_{(i)}+k_{(i)}\etr\|_2},
\end{align}
where
\begin{align*}
    k_{(i)}=\frac{\mathbf{B}_{(i)}^T(-\etr)}{\|\etr\|_2^2}-\frac{\sqrt{1-\left(\frac{\mathbf{B}_{(i)}^T(-\etr)}{\|\etr\|_2}\right)^2}}{\sqrt{1-\left(\frac{\mathbf{B}_{(q)}^T(-\etr)}{\|\etr\|_2}\right)^2}}\frac{\mathbf{B}_{(q)}^T(-\etr)}{\|\etr\|_2^2}\geq 0.
\end{align*}
\end{lemma}
\begin{proof}
Using Eq. \eqref{eq.temp_ortho} again, we decompose $\mathbf{B}_{(i)}$ into two parts: one in the direction of $(-\etr)$, the other orthogonal to $(-\etr)$.
\begin{align*}
    \mathbf{B}_{(i)}=\frac{\mathbf{B}_{(i)}^T(-\etr)}{\|\etr\|_2^2}(-\etr)+\left(\mathbf{B}_{(i)}-\frac{\mathbf{B}_{(i)}^T(-\etr)}{\|\etr\|_2^2}(-\etr)\right).
\end{align*}
Thus, we have
\begin{align*}
    \mathbf{B}_{(i)}+k_{(i)}\etr=&\frac{\sqrt{1-\left(\frac{\mathbf{B}_{(i)}^T(-\etr)}{\|\etr\|_2}\right)^2}}{\sqrt{1-\left(\frac{\mathbf{B}_{(q)}^T(-\etr)}{\|\etr\|_2}\right)^2}}\frac{\mathbf{B}_{(q)}^T(-\etr)}{\|\etr\|_2^2}(-\etr)\\
    &+\left(\mathbf{B}_{(i)}-\frac{\mathbf{B}_{(i)}^T(-\etr)}{\|\etr\|_2^2}(-\etr)\right).
\end{align*}
We then have
\begin{align*}
    &\frac{\sqrt{1-\left(\frac{\mathbf{B}_{(q)}^T(-\etr)}{\|\etr\|_2}\right)^2}}{\sqrt{1-\left(\frac{\mathbf{B}_{(i)}^T(-\etr)}{\|\etr\|_2}\right)^2}}\cdot(\mathbf{B}_{(i)}+k_{(i)}\etr)\\
    =& \frac{\sqrt{1-\left(\frac{\mathbf{B}_{(q)}^T(-\etr)}{\|\etr\|_2}\right)^2}}{\sqrt{1-\left(\frac{\mathbf{B}_{(i)}^T(-\etr)}{\|\etr\|_2}\right)^2}}\cdot \left(\mathbf{B}_{(i)}-\frac{\mathbf{B}_{(i)}^T(-\etr)}{\|\etr\|_2^2}(-\etr)\right)\nonumber\\
    &+\frac{\mathbf{B}_{(q)}^T(-\etr)}{\|\etr\|_2^2}(-\etr)\\
    =&\mathbf{C}_{(i)}.
\end{align*}
In other words, $\mathbf{C}_{(i)}$ and $\mathbf{B}_{(i)}+k_{(i)}\etr$ are along the same direction. Since $\|\mathbf{C}_{(i)}\|_2=1$, it must then also be equal to a normalized version of $\mathbf{B}_{(i)}+k_{(i)}\etr$, i.e.,
\begin{align*}
    \frac{\mathbf{B}_{(i)}+k_{(i)}\etr}{\|\mathbf{B}_{(i)}+k_{(i)}\etr\|_2}=\mathbf{C}_{(i)}.
\end{align*}
This verifies \eqref{eq.temp_0122}. Note that $\mathbf{C}_{(i)}\etr=\mathbf{B}_{(q)}\etr\leq 0$ by Lemma \ref{le.C_is_one}.
It then only remains to prove $k_{(i)}\geq 0$. Towards this end, because of Eq. \eqref{eq.sorted_B}, we have
\begin{align*}
    &\mathbf{B}_{(q)}^T(-\etr)\leq \mathbf{B}_{(i)}^T(-\etr)\\
    \implies & \frac{\sqrt{1-\left(\frac{\mathbf{B}_{(i)}^T(-\etr)}{\|\etr\|_2}\right)^2}}{\sqrt{1-\left(\frac{\mathbf{B}_{(q)}^T(-\etr)}{\|\etr\|_2}\right)^2}}\leq 1.
\end{align*}
Thus, we have
\begin{align*}
    k_{(i)}\geq \frac{\mathbf{B}_{(i)}^T(-\etr)}{\|\etr\|_2^2}-\frac{\mathbf{B}_{(q)}^T(-\etr)}{\|\etr\|_2^2}\geq 0.
\end{align*}
The result of the lemma thus follows.
\end{proof}

Combining Lemma \ref{le.relax_is_relax} and Lemma \ref{le.C_has_k}, we have proven that if $\lambda^T(-\etr)\geq \mathbf{B}_{(1)}^T$ and $\lambda^T\mathbf{B}_{(i)}\leq 1$, then $\lambda^T\mathbf{C}_{(i)}\leq 1$. Therefore, we have shown step 2, i.e., any feasible $\lambda$ for the problem \eqref{prob.with_greater_than_B1} is also feasible for the problem \eqref{prob.last_relax}. The conclusion of Lemma \ref{le.from_B_to_C} thus follows.

\subsection{Proof of Lemma \ref{le.that_is_q_th}}\label{app.proof_lemma_10}

First, we show that $\lambda_*$ defined in the lemma is feasible for the problem \eqref{prob.last_relax}. Towards this end, note that because $\mathbf{C}_{(i)}^T(-\etr)=\mathbf{B}_{(q)}^T(-\etr)$ (see Lemma \ref{le.C_is_one}) for all $i\in\{1,2,\cdots, q\}$, we have $\lambda_*^T\mathbf{C}_{(i)}=1$, which implies that $\lambda_*$ is feasible for the problem \eqref{prob.last_relax}. Then, it remains to show that $\lambda_*$ is optimal for the problem \eqref{prob.last_relax} with probability at least $1-e^{-q/4-n}$. 

Next, we will define an event $\mathscr{A}$ with probability no smaller than
\begin{align}\label{eq.temp_0125_1}
    1-2^{-q+1}\sum_{i=0}^{n-2}\binom{q-1}{i},
\end{align}
such that $\lambda^*$ is optimal whenever event $\mathscr{A}$ occurs.
Towards this end, consider the null space of $-\etr$, which is defined as
\begin{align*}
    \ker(-\etr):=\{\lambda\ \big|\ \lambda^T(-\etr)=0\}.
\end{align*}
We then decompose all $\mathbf{C}_{(i)}$'s into two components, one is in the direction of $-\etr$, the other is in the null space of $-\etr$. Specifically, we have
\begin{align}
     \mathbf{C}_{(i)}=&\left(\mathbf{C}_{(i)}-\frac{\mathbf{C}_{(i)}^T(-\etr)}{\|\etr\|_2^2}(-\etr)\right)+\frac{\mathbf{C}_{(i)}^T(-\etr)}{\|\etr\|_2^2}(-\etr)\nonumber\\
     =&\left(\mathbf{C}_{(i)}-\frac{\mathbf{C}_{(q)}^T(-\etr)}{\|\etr\|_2^2}(-\etr)\right)+\frac{\mathbf{C}_{(q)}^T(-\etr)}{\|\etr\|_2^2}(-\etr),\label{eq.temp_0131_1}
\end{align}
where in the last step we have used $\mathbf{C}_{(i)}^T(-\etr)=\mathbf{C}_{(q)}^T(-\etr)$.
For conciseness, we define
\begin{align*}
    \mathbf{D}_{(i)}\defeq\mathbf{C}_{(i)}-\frac{\mathbf{C}_{(q)}^T(-\etr)}{\|\etr\|_2^2}(-\etr).
\end{align*}
Since $\|\mathbf{C}_{(i)}\|_2=1$ and $\mathbf{C}_{(i)}$ is orthogonal to $\mathbf{C}_{(i)}-\mathbf{D}_{(i)}$, we have
\begin{align*}
    \|\mathbf{D}_{(i)}\|_2=\sqrt{\|\mathbf{C}_{(i)}\|_2^2-\|\mathbf{C}_{(i)}-\mathbf{D}_{(i)}\|_2^2}=\sqrt{1-\left(\mathbf{C}_{(q)}^T(-\etr)\right)^2}.
\end{align*}
Thus, $\mathbf{D}_{(i)}$ has the same $\ell_2$-norm for all $i\in\{1,\cdots,q\}$. Therefore, $\mathbf{D}_{(1)},\mathbf{D}_{(2)},\cdots,\mathbf{D}_{(q)}$ can be viewed as $q$ points in a sphere in the space $\ker(-\etr)$, which has $(n-1)$ dimensions. By Lemma \ref{le.C_has_k}, we know that the projections of $\mathbf{C}_{(i)}$ and $\mathbf{B}_{(i)}$ to the space $\ker(-\etr)$ have the same direction. Because $\mathbf{B}_{(i)}$'s are uniformly distributed on the hemisphere in $\mathds{R}^n$, their projections to $\ker(-\etr)$ are also uniformly distributed. Therefore, $\mathbf{D}_{(i)}$'s are uniformly distributed on a $(n-1)$-dim sphere.
By Lemma \ref{le.hemisphere}, with probability \eqref{eq.temp_0125_1}, there exists at least one of the vectors $\mathbf{D}_{(1)},\mathbf{D}_{(2)},\cdots,\mathbf{D}_{(q)}$ in any hemisphere. Let $\mathscr{A}$ denote this event with probability \eqref{eq.temp_0125_1}. Note that if we use a vector $\gamma\in \ker(-\etr)$ to represent the axis of any such hemisphere in $\mathbf{R}^{n-1}$, then whether a vector $\zeta\in \ker(-\etr)$ is on that hemisphere is totally determined by checking whether $\gamma^T\zeta> 0$. Thus, the event $\mathscr{A}$ is equivalent to,
for any $\gamma\in \ker(-\etr)$, there exists at least one of the vectors $\mathbf{D}_{(1)},\mathbf{D}_{(2)},\cdots,\mathbf{D}_{(q)}$ such that its inner product with $\gamma$ is positive. 

We now prove the following statement that $\lambda^*$ is optimal whenever event $\mathscr{A}$ occurs.
We prove by contradiction. Assume that event $\mathscr{A}$ occurs, suppose on the contrary that the maximum point is achieved at $\lambda=\mu\neq \lambda_*$ such that $\mu^T(-\etr)>(\lambda^*)^T(-\etr)$. Since $\mu$ meets all constraints, we have
\begin{align}\label{eq.pp_contradiction}
    (\mu-\lambda_*)^T\mathbf{C}_{(i)}= \mu^T\mathbf{C}_{(i)}-1\leq 0\ \forall i\in\{1,\cdots,q\}.
\end{align}
Comparing the objective values at $\mu$ and $\lambda_*$, we have
\begin{align}\label{eq.pp_better}
    (\mu-\lambda_*)^T(-\etr)>0.
\end{align}
Similar to the decomposition of $\mathbf{C}_{(i)}$ in Eq. \eqref{eq.temp_0131_1}, we decompose $(\mu-\lambda_*)$ into two components: one in the direction of $-\etr$ and the other in the null space of $-\etr$. Specifically, we have
\begin{align*}
    (\mu-\lambda_*)=&\left((\mu-\lambda_*)-\frac{(\mu-\lambda_*)^T(-\etr)}{\|\etr\|_2^2}(-\etr)\right)\\
    &+\frac{(\mu-\lambda_*)^T(-\etr)}{\|\etr\|_2^2}(-\etr).
\end{align*}
Thus, we have
\begin{align*}
    &(\mu-\lambda_*)^T\mathbf{C}_{(i)}\\
    =&\left((\mu-\lambda_*)-\frac{(\mu-\lambda_*)^T(-\etr)}{\|\etr\|_2^2}(-\etr)\right)^T\\
    &\cdot\left(\mathbf{C}_{(i)}-\frac{\mathbf{C}_{(q)}^T(-\etr)}{\|\etr\|_2^2}(-\etr)\right)\\
    &+\frac{1}{\|\etr\|_2^2}\left((\mu-\lambda_*)^T(-\etr)\right)\left(\mathbf{C}_{(q)}^T(-\etr)\right).
\end{align*}
For conciseness, we define
\begin{align*}
    \delta\defeq(\mu-\lambda_*)-\frac{(\mu-\lambda_*)^T(-\etr)}{\|\etr\|_2^2}(-\etr).
\end{align*}
We then have
\begin{align}
    (\mu-\lambda_*)^T\mathbf{C}_{(i)}&=\delta^T\mathbf{D}_{(i)}+\frac{1}{\|\etr\|_2^2}\left((\mu-\lambda_*)^T(-\etr)\right)\left(\mathbf{C}_{(q)}^T(-\etr)\right)\geq\delta^T\mathbf{D}_{(i)},\label{eq.temp_0125_2}
\end{align}
where the last inequality holds because $(\mu-\lambda_*)^T(-\etr)>0$ (by Eq. \eqref{eq.pp_better}) and $\mathbf{C}_{(q)}^T(-\etr) = \mathbf{B}_{(q)}^T(-\etr) \allowbreak\geq 0$ (by Lemma \ref{le.C_is_one} and Eq. \eqref{eq.sorted_B}).
Since $\delta\in \ker(-\etr)$ and event $\mathscr{A}$ occurs, we can therefore find a $\mathbf{D}_{(k)}$ such that $\delta^T\mathbf{D}_{(k)}> 0$. Letting $i=k$ in Eq. \eqref{eq.temp_0125_2}, we then have
\begin{align*}
    (\mu-\lambda_*)^T\mathbf{C}_{(k)}\geq \delta^T\mathbf{D}_{(k)}>0,
\end{align*}
which contradicts Eq. (\ref{eq.pp_contradiction}). Therefore, $\lambda^*$ must be optimal whenever event $\mathscr{A}$ occurs.

It only remains to show that the probability of event $\mathscr{A}$ given in Eq. \eqref{eq.temp_0125_1} is at least $1-e^{-(q/4-n)}$, which is proven in the following Lemma \ref{le.estimate_prob_q}.
\begin{lemma}\label{le.estimate_prob_q}
\begin{align*}
    1-2^{-q+1}\sum_{i=0}^{n-2}\binom{q-1}{i}\geq 1-e^{-(q/4-n)}.
\end{align*}
\end{lemma}
The proof of Lemma \ref{le.estimate_prob_q} uses the following Chernoff bound.
\begin{lemma}[Chernoff bound for binomial distribution, Theorem 4(ii) in \citep{goemans2015chernoff}]\label{le.binomial_chernoff}
Let $X$ be a random variable that follows the binomial distribution $B(m,\overline{p})$, where $m$ denotes the number of experiments and $\overline{p}$ denotes the probability of success for each experiment. Then
\begin{align*}
    \prob{X\leq (1-\delta)m\overline{p}}\leq \exp\left(-\frac{\delta^2m\overline{p}}{2}\right) \ \forall \delta\in(0,1).
\end{align*}
\end{lemma}
\begin{myproof}{Lemma \ref{le.estimate_prob_q}}
Consider a random variable $X$ with binomial distribution $B(q-1,\ 1/2)$. We have
\begin{align*}
    \prob{X\leq n-2}=2^{-q+1}\sum_{i=0}^{n-2}\binom{q-1}{i}.
\end{align*}
Let
\begin{align*}
    \delta = 1-\frac{2(n-2)}{q-1},\quad \text{i.e., }\quad 1-\delta=\frac{2(n-2)}{q-1}.
\end{align*}
Applying Chernoff bound stated in the Lemma \ref{le.binomial_chernoff}, we have
\begin{align*}
    \prob{X\leq n-2}&=\prob{X\leq \left(1-\delta\right)\frac{q-1}{2}}\\
    &\leq e^{-\delta^2(q-1)/4}.
\end{align*}
Also, we have
\begin{align*}
    \delta^2(q-1)/4&=\frac{1}{4}\left(1-\frac{2(n-2)}{q-1}\right)^2(q-1)\\
    &\geq\frac{1}{4}\left(1-\frac{4(n-2)}{q-1}\right)(q-1)\\
    &= \frac{1}{4}(q-1-4(n-2))\\
    &\geq \frac{q}{4}-n.
\end{align*}
Thus, we have
\begin{align*}
    1-2^{-q+1}\sum_{i=0}^{n-2}\binom{q-1}{i}&=1-\prob{x\leq n-2} \\
    &\geq 1-e^{-\delta^2(q-1)/4}\\
    &\geq 1-e^{-(q/4-n)}.
\end{align*}
\end{myproof}

\subsection{Proof of Lemma \ref{le.estimate_Aq}}\label{app.proof_of_lemma_12}

The proof consists of three steps. Recall that $\mathbf{B}_{(5n)}^T(-\etr)$ ranks the $5n$-th among all $\mathbf{A}_i^T(-\etr)$'s and $\mathbf{A}_i^T\etr$'s. In step 1, we first estimate the probability distribution about $\mathbf{A}_i^T(-\etr)$. In step 2, we use the result in step 1 to estimate $\mathbf{B}_{5n}^T(-\etr)$. In step 3, we relax and simplify the result in step 2 to get the exact result of Lemma \ref{le.estimate_Aq}.
Without loss of generality\footnote{Rotating $\etr$ around the origin is equivalent to rotating all columns of $\mathbf{A}$. Since the distribution of $\mathbf{A}_i$ is uniform on the unit hyper-sphere in $\mathds{R}^n$, such rotation does not affect the objective of the problem (\ref{prob.original}).}, we let $\etr=[-\|\etr\|_2\ \ 0\ \ \cdots\ \ 0]^T$.
Thus, $\mathbf{A}_i^T(-\etr)=\|\etr\|_2\mathbf{A}_{i1}$, where $\mathbf{A}_{ij}$ denotes the $j$-th element of the $i$-th column of $\mathbf{A}$. 

\subsection*{Step 1}
Notice that $\mathbf{A}_i$ (i.e., the $i$-th column of $\mathbf{A}$) is a normalized Gaussian random vector. We use $\mathbf{A}_i'$ to denote the standard Gaussian random vector before the normalization, i.e., $\mathbf{A}_i'$ is a $n\times 1$ vector where each element follows i.i.d. standard Gaussian distribution. Thus, we have
\begin{align*}
    |\mathbf{A}_{i1}|=\frac{|\mathbf{A}_{i1}'|}{\|\mathbf{A}_i'\|_2}=\frac{|\mathbf{A}_{i1}'|}{\sqrt{(\mathbf{A}_{i1}')^2+\sum_{j=2}^n(\mathbf{A}_{ij}')^2}}.
\end{align*}
For any $k>1$, we then have
\begin{align}
    &\prob{\frac{1}{|\mathbf{A}_{i1}|}\leq k}=\prob{(\mathbf{A}_{i1}')^2\geq \frac{\sum_{j=2}^n(\mathbf{A}_{ij}')^2}{k^2-1}}.\label{eq.temp_0124_2}
\end{align}
Notice that $\sum_{j=2}^n(\mathbf{A}_{ij}')^2$ follows the chi-square distribution with $(n-1)$ degrees of freedom. When $n$ is large, $\sum_{j=2}^n(\mathbf{A}_{ij}')^2$ should be around its mean value. Further, $\mathbf{A}_{i1}'$ follows standard Gaussian distribution. Next, we use results of chi-square distribution and Gaussian distribution to estimate the distribution of $\mathbf{A}_{i1}$. The following lemma is useful for approximating a Gaussian distribution.

\begin{lemma}\label{le.estimate_phi_c}
When $t\geq 0$, we have
\begin{align*}
    \frac{\sqrt{2/\pi}\ e^{-t^2/2}}{t+\sqrt{t^2+4}}\leq \Phi^c(t)\leq \frac{\sqrt{2/\pi}\ e^{-t^2/2}}{t+\sqrt{t^2+\frac{8}{\pi}}},
\end{align*}
where $\Phi^c(\cdot)$ denotes the complementary cumulative distribution function (cdf) of standard Gaussian distribution, i.e.,
\begin{align*}
    \Phi^c(t)=\frac{1}{\sqrt{2\pi}}\int_t^\infty e^{-u^2/2}du.
\end{align*}
\end{lemma}
\begin{proof}
By (7.1.13) in \citep{abramowitz1972handbook}, we know that
\begin{align*}
    \frac{1}{x+\sqrt{x^2+2}}\leq e^{x^2}\int_x^\infty e^{-y^2}dy\leq \frac{1}{x+\sqrt{x^2+\frac{4}{\pi}}}\quad (x\geq 0).
\end{align*}
Let $x=t/\sqrt{2}$. We have
\begin{align*}
    &\frac{1}{\frac{t}{\sqrt{2}}+\sqrt{\frac{t^2}{2}+2}}\leq e^{t^2/2}\int_{\frac{t}{\sqrt{2}}}^\infty e^{-y^2}dy\leq \frac{1}{\frac{t}{\sqrt{2}}+\sqrt{\frac{t^2}{2}+\frac{4}{\pi}}}\\
    \implies & \frac{\sqrt{2/\pi}\ e^{-t^2/2}}{t+\sqrt{t^2+4}}\leq \frac{1}{\sqrt{\pi}}\int_{\frac{t}{\sqrt{2}}}^\infty e^{-y^2}dy\leq \frac{\sqrt{2/\pi}\ e^{-t^2/2}}{t+\sqrt{t^2+\frac{8}{\pi}}}\\
    \implies & \frac{\sqrt{2/\pi}\ e^{-t^2/2}}{t+\sqrt{t^2+4}}\leq \frac{1}{\sqrt{2\pi}}\int_{t}^\infty e^{-\frac{z^2}{2}}dz\leq \frac{\sqrt{2/\pi}\ e^{-t^2/2}}{t+\sqrt{t^2+\frac{8}{\pi}}}\text{ (let $z\defeq \sqrt{2}y$)}\\
    \implies & \frac{\sqrt{2/\pi}\ e^{-t^2/2}}{t+\sqrt{t^2+4}}\leq \Phi^c(t)\leq \frac{\sqrt{2/\pi}\ e^{-t^2/2}}{t+\sqrt{t^2+\frac{8}{\pi}}}.
\end{align*}
The result of this lemma thus follows.
\end{proof}
The following lemma gives an estimate of the probability distribution of $\mathbf{A}_{i1}$.
\begin{lemma}
\begin{align}\label{eq.prob_temp_bound}
    \prob{\frac{1}{|\mathbf{A}_{i1}|}\leq k}\geq 2\left(1-\frac{1}{\sqrt{e}}\right)\sqrt{\frac{2}{\pi}}\frac{e^{-t^2/2}}{t+\sqrt{t^2+4}},
\end{align}
where
\begin{align*}
    t=\sqrt{\frac{n+\sqrt{2}\sqrt{n-1}}{k^2-1}}.
\end{align*}
\end{lemma}
\begin{proof}
For any $m>0$, we have
\begin{align}
    &\prob{\frac{1}{|\mathbf{A}_{i1}|}\leq k}=\prob{(\mathbf{A}_{i1}')^2\geq \frac{\sum_{j=2}^n(\mathbf{A}_{ij}')^2}{k^2-1}}\nonumber\\
    \geq&\prob{(\mathbf{A}_{i1}')^2\geq \frac{n-1+2\sqrt{(n-1)m}+2m}{k^2-1}}\nonumber\\
    &\cdot\prob{\sum_{j=2}^n(\mathbf{A}_{ij}')^2\leq n-1+2\sqrt{(n-1)m}+2m}\text{ (since all $\mathbf{A}_{ij}'$'s are \emph{i.i.d.})}\nonumber
\end{align}
Notice that $\sum_{j=2}^n(\mathbf{A}_{ij}')^2$ follows chi-square distribution with $(n-1)$ degrees freedom. Applying Lemma \ref{le.chi_bound}, we have
\begin{align}
     &\prob{\frac{1}{|\mathbf{A}_{i1}|}\leq k}\nonumber\\
     \geq&\prob{(\mathbf{A}_{i1}')^2\geq \frac{n-1+2\sqrt{(n-1)m}+2m}{k^2-1}}\cdot(1-e^{-m})\nonumber\\
    =&2(1-e^{-m})\Phi^c\left(\sqrt{\frac{n-1+2\sqrt{(n-1)m}+2m}{k^2-1}}\right)\label{eq.temp_bound_1}\\
    &\text{ (since the distribution of $\mathbf{A}_{i1}$ is symmetric with respect to $0$)}.\nonumber
\end{align}
We now let $m=1/2$ in Eq. (\ref{eq.temp_bound_1}). Then
\begin{align*}
    \sqrt{\frac{n-1+2\sqrt{(n-1)m}+2m}{k^2-1}}=\sqrt{\frac{n+\sqrt{2(n-1)}}{k^2-1}}=t.
\end{align*}
Applying Lemma \ref{le.estimate_phi_c}, the result of this lemma thus follows.
\end{proof}

\subsection*{Step 2}
Next, we estimate the distribution of $\mathbf{B}_{(5n)}^T(-\etr)$. We first introduce a lemma below, which will be used later.
\begin{lemma}\label{le.t_0_5}
If $t\geq 0.5$, then $t+\sqrt{t^2+4}<e^{t+0.5}$.
\end{lemma}
\begin{proof}
Let $f(t)=e^{t+0.5}-(t+\sqrt{t^2+4})$. Then $f(0.5)\approx0.157>0$. We only need to prove that $df/dt\geq 0$ when $t\geq 0.5$. Indeed, when $t\geq 0.5$, we have
\begin{align*}
    &\frac{df(t)}{dt}=e^{t+0.5}-1-\frac{t}{\sqrt{t^2+4}}\geq e-1-1\geq 0 \text{ (notice that $t\leq \sqrt{t^2+4}$ for any $t$)}.
\end{align*}
\end{proof}
Now, we estimate $\mathbf{B}_{(5n)}^T(-\etr)$ by the following proposition.
\begin{proposition}\label{prop.bound_WI1}
Let
\begin{align}\label{eq.define_constant_C}
    C=\frac{1}{5}\left(1-\frac{1}{\sqrt{e}}\right)\sqrt{\frac{2}{\pi}}\approx0.063.
\end{align}
When $p-s\geq ne^{9/8}/C$, the following holds.
\begin{align}\label{eq.temp_0122_2}
    \frac{\|\etr\|_2}{\mathbf{B}_{(5n)}^T(-\etr)}\leq \sqrt{1+\frac{n+\sqrt{2}\sqrt{n-1}}{\left(\sqrt{2\ln\frac{C(p-s)}{n}}-1\right)^2}},
\end{align}
with probability at least $1-e^{-5n/4}$.
\end{proposition}
(Notice that, by applying this proposition in Corollary \ref{coro.bound_by_Aq}, Eq.~\eqref{eq.temp_0122_2} already suggests an upper bound of $\|w^I\|_1$.)
\begin{proof}
For conciseness, we use $\rho(n,k)$ to denote the right-hand-side of Eq. (\ref{eq.prob_temp_bound}), i.e.,
\begin{align*}
    \rho(n,k)=10C\frac{e^{-t^2/2}}{t+\sqrt{t^2+4}}\bigg|_{t=\sqrt{\frac{n+\sqrt{2}\sqrt{n-1}}{k^2-1}}}.
\end{align*}
Let $k$ take the value of the $\text{RHS}$ of Eq. \eqref{eq.temp_0122_2}. Then, we have
\begin{align}
    t=&\sqrt{\frac{n+\sqrt{2(n-1)}}{k^2-1}}\nonumber\\
    =&\sqrt{\frac{n+\sqrt{2(n-1)}}{1+\frac{n+\sqrt{2(n-1)}}{\left(2\sqrt{\ln \frac{C(p-s)}{n}}-1\right)^2}-1}}\nonumber\\
    =&\sqrt{2\ln \frac{C(p-s)}{n}}-1.\label{eq.temp_0129_1}
\end{align}
Because $p-s\geq ne^{9/8}/C$, we have
$t\geq 0.5$. By Lemma \ref{le.t_0_5}, we have $t+\sqrt{t^2+4}<e^{t+0.5}$. Thus, we have
\begin{align}
    \rho(n,k)&\geq 10C \exp\left(-\frac{t^2}{2}-t-0.5\right)\nonumber\\
    &=10C\exp\left(-\frac{1}{2}(t+1)^2\right)\nonumber\\
    &=10C\frac{n}{C(p-s)}\quad\text{ (using Eq. \eqref{eq.temp_0129_1})}\nonumber\\
    &=\frac{10n}{p-s}.\label{eq.temp_0123}
\end{align}
By the definition of $\mathbf{B}_{(5n)}$ and Eq. \eqref{eq.sorted_B}, we have
\begin{align}\label{eq.1217_1}
    \prob{\text{Eq. \eqref{eq.temp_0122_2}}}=\prob{\#\{i\mid i\in\{1,2,\cdots,p-s\}, \frac{1}{|\mathbf{A}_{i1}|}\leq k\}\geq 5n}.
\end{align}
Consider a random variable $x$ following the binomial distribution $\mathcal{B}(p-s, \rho(n,k))$. Since $\mathbf{A}_{i1}$'s are \emph{i.i.d.} and $\prob{\frac{1}{|\mathbf{A}_{i1}|}\leq k}\geq \rho(n,k)$, we must have
\begin{align*}
    \text{Eq. (\ref{eq.1217_1})}\geq \prob{x\geq 5n}= 1-\prob{x\leq 5n-1}\geq 1-\prob{x\leq 5n}.
\end{align*}
It only remains to show that $\prob{x\leq 5n}\leq e^{-5n/4}$.
Applying Lemma \ref{le.binomial_chernoff}, we have
\begin{align}
    \prob{x\leq 5n}=&\prob{x\leq (1-\delta)(p-s)\rho(n,k)}\nonumber\\
    \leq &e^{-\delta^2 (p-s)\rho(n,k)/2},\label{eq.temp_0129_2}
\end{align}
where
\begin{align*}
    \delta&=1-\frac{5n}{(p-s)\rho(n,k)}\quad \text{  (so $5n=(1-\delta)(p-s)\rho(n,k)$)}.
\end{align*}
Since $(p-s)\rho(n,k)\geq 10n$ by Eq. \eqref{eq.temp_0123}, we must have $\delta\geq 0.5$. Substituting into Eq. \eqref{eq.temp_0129_2}, we have $\prob{x\leq 5n}\leq \exp(-0.5^2\cdot(10n)/2)=e^{-5n/4}$.
\end{proof}

\subsection*{Step 3}
Notice that by utilizing Proposition \ref{prop.bound_WI1} and Corollary \ref{coro.bound_by_Aq}, we already have an upper bound on $\|w^I\|_1$. To get the simpler form in Lemma \ref{le.estimate_Aq}, we only need to use the following lemma to simplify the expression in Proposition \ref{prop.bound_WI1}.
\begin{lemma}
When $n\geq 100$ and $p\geq (16n)^4$, we must have
\begin{align*}
    \text{RHS of Eq. \eqref{eq.temp_0122_2}} \leq \sqrt{1+\frac{3n/2}{\ln p}}.
\end{align*}
\end{lemma}
\begin{proof}
Because $n>100$ and $p\geq (16n)^4$, we have $p\geq 10^{12}$.
Thus, we ahave
\begin{align*}
    &\ln p\geq 25\text{ (since $\ln 10\approx2.3>25/12$)}\\
    \implies & \sqrt{\ln p}-2\geq 3\\
    \implies & \sqrt{\ln p}-2\geq \sqrt{3\ln 2+6}\text{ (since $\ln 2<1$)}\\
    \implies & \frac{1}{2}\left(\sqrt{\ln p}-2\right)^2\geq \frac{3}{2}\ln 2+3\\
    \implies & \frac{3}{2}(\ln p- \ln 2)\geq \ln p+2\sqrt{\ln p}+1\text{ (by expanding the square and rearranging terms)}\\
    \implies & \sqrt{\ln p}+1\leq \sqrt{\frac{3}{2}}\sqrt{\ln p-\ln 2}\quad\text{ (by taking square root on both sides)}.
\end{align*}
Because $s\leq n$ and $p\geq (16n)^4\geq 2n$, we have $\ln (p-s)\geq \ln (p-n)\geq \ln (p/2)$. Thus, we have
\begin{align}\label{eq.temp0110}
    \sqrt{\ln p}+1\leq \sqrt{\frac{3}{2}}\sqrt{\ln (p-s)}.
\end{align}
We still use $C$ defined in Eq. (\ref{eq.define_constant_C}). We have
\begin{align}
    &p\geq (16n)^4 \implies p\geq \left(\frac{n}{C}\right)^4+n+\left((16n)^4-\left(\frac{n}{C}\right)^4-n\right).\label{eq.temp_0129_3}
\end{align}
Note that
\begin{align*}
    (16n)^4-\left(\frac{n}{C}\right)^4-n=&n\left(n^3\left(16^4-\left(\frac{1}{C}\right)^4\right)-1\right)\\
    \geq & n\left(n^3-1\right)\text{ (because $16^4-\left(\frac{1}{C}\right)^4\approx16^4-\left(\frac{1}{0.063}\right)^4>1$)}\\
    \geq &0\text{ (because $n\geq 1$)}.
\end{align*}
Applying it in Eq. \eqref{eq.temp_0129_3}, we have
\begin{align}
    & p-n\geq \left(\frac{n}{C}\right)^4\nonumber\\
    \implies & p-s\geq \left(\frac{n}{C}\right)^4\text{ (because $s\leq n$)}\nonumber\\ \implies &(p-s)^{-3}\left(\frac{C}{n}\right)^4(p-s)^4\geq 1\nonumber\\
    \implies &-3\ln (p-s) + 4\ln \frac{C(p-s)}{n}\geq 0\nonumber\\
    \implies &2\ln \frac{C(p-s)}{n}\geq \frac{3}{2}\ln (p-s)\nonumber\\
    \implies&2\ln \frac{C(p-s)}{n}\geq (\sqrt{\ln p}+1)^2\text{ (by Eq. (\ref{eq.temp0110}))}\nonumber\\
    \implies &\left(\sqrt{2\ln\frac{C(p-s)}{n}}-1\right)^2\geq \ln p.\label{eq.temp_0129_4}
\end{align}
When $n\geq 100$, we always have
\begin{align}
    &n-1\leq \frac{n^2}{8}\nonumber\\
    \implies &\sqrt{2}\sqrt{n-1}\leq \frac{n}{2}.\label{eq.temp_0129_5}
\end{align}
Substituting Eq. \eqref{eq.temp_0129_4} and Eq. \eqref{eq.temp_0129_5} into the RHS of Eq. \eqref{eq.temp_0122_2}, the conclusion of this lemma thus follows.
\end{proof}

\section{Proof of Proposition \ref{prop.M} (upper bound of \texorpdfstring{$M$}{M})}\label{app.M}
For conciseness, we define $G_{ij}\defeq \mathbf{X}_i^T\mathbf{X}_j$. According to the normalization in Eq. \eqref{def.norm_X}, we have
\begin{align*}
    G_{ij}\defeq\frac{\mathbf{H}_i^T\mathbf{H}_j}{\|\mathbf{H}_i\|_2\|\mathbf{H}_j\|_2}.
\end{align*}
Our proof consists of four steps. In step 1, we relate the tail probability of any $|G_{ij}|$ (where $i\neq j$) to the tail probability of $\mathbf{H}_i^T\mathbf{H}_j$. In step 2, we estimate the tail probability of $\mathbf{H}_i^T\mathbf{H}_j$. In step 3, we use union bound to estimate the cdf of $M$, so that we can get an upper bound on $M$ with high probability. In step 4, we simplify the result derived in step 3.

\subsection*{Step 1: Relating the tail probability of $|G_{ij}|$ to that of $\mathbf{H}_i^T\mathbf{H}_j$.}
For any $i\neq j$, we have
\begin{align}
    &\prob{|G_{ij}|> a}\nonumber\\
    &=\prob{|G_{ij}|>a,\|\mathbf{H}_i\|_2\geq \sqrt{\frac{n}{2}},\|\mathbf{H}_j\|_2\geq \sqrt{\frac{n}{2}}}\nonumber\\
    &\quad \ +\prob{|G_{ij}|>a,\left(\|\mathbf{H}_i\|_2< \sqrt{\frac{n}{2}}\text{ or }\|\mathbf{H}_j\|_2<\sqrt{\frac{n}{2}}\right)}.\label{eq.temp_0129_6}
\end{align}
The first term can be bounded by
\begin{align*}
    \prob{|G_{ij}|>a,\|\mathbf{H}_i\|_2\geq \sqrt{\frac{n}{2}},\|\mathbf{H}_j\|_2\geq \sqrt{\frac{n}{2}}}\leq \prob{|\mathbf{H}_i^T\mathbf{H}_j|>\frac{na}{2}},
\end{align*}
because
\begin{align*}
    |G_{ij}|>a,\|\mathbf{H}_i\|_2\geq \sqrt{\frac{n}{2}},\|\mathbf{H}_j\|_2\geq \sqrt{\frac{n}{2}}\implies |\mathbf{H}_i^T\mathbf{H}_j|>\frac{na}{2}.
\end{align*}
Thus, we have, from Eq. \eqref{eq.temp_0129_6},
\begin{align}
    \prob{|G_{ij}|> a}&\leq \prob{|\mathbf{H}_i^T\mathbf{H}_j|>\frac{na}{2}}+\prob{\|\mathbf{H}_i\|_2<\sqrt{\frac{n}{2}}}\nonumber\\
    &\quad+\prob{\|\mathbf{H}_j\|_2<\sqrt{\frac{n}{2}}}\label{eq.temp_0129_7}\\
    &=2\prob{\mathbf{H}_i^T\mathbf{H}_j>\frac{na}{2}}+2\prob{\|\mathbf{H}_i\|_2<\sqrt{\frac{n}{2}}},\nonumber
\end{align}
where the last equality is because the distribution of $\mathbf{H}_i^T\mathbf{H}_j$ is symmetric around $0$, and $\mathbf{H}_j$ has the same distribution as $\mathbf{H}_i$.
Notice that $\|\mathbf{H}_i\|_2^2$ follows chi-square distribution with $n$ degrees of freedom. By Lemma \ref{le.chi_bound} (using $x=n/16$), we have
\begin{align*}
    \prob{\|\mathbf{H}_i\|_2<\sqrt{\frac{n}{2}}}=\prob{\|\mathbf{H}_i\|_2^2<\frac{n}{2}}\leq e^{-n/16}.
\end{align*}
Thus, we have
\begin{align}\label{eq.mid_result_M}
    \prob{|G_{ij}|> a}\leq 2\prob{\mathbf{H}_i^T\mathbf{H}_j>\frac{na}{2}}+2e^{-n/16}.
\end{align}

\subsection*{Step 2: Estimating the tail probability of $\mathbf{H}_i^T\mathbf{H}_j$.}
Notice that $\mathbf{H}_i^T\mathbf{H}_j$ is the sum of product of two Gaussian random variables. We will use the Chernoff bound to estimate its tail probability. Towards this end, we first calculate the moment generating function (M.G.F) of the product of two Gaussian random variables. 
\begin{lemma}\label{le.MGF}
If $X$ and $Y$ are two independent standard Gaussian random variables, then the M.G.F of $XY$ is
\begin{align*}
    \mathds{E}[e^{tXY}]=\frac{1}{\sqrt{1-t^2}},
\end{align*}
for any $t^2<1$.
\end{lemma}
\begin{proof}
\begin{align*}
    &\mathds{E}[e^{tXY}]\\
    =&\frac{1}{2\pi}\int_{-\infty}^\infty\int_{-\infty}^\infty e^{txy}e^{-\frac{x^2+y^2}{2}}dxdy\\
    =&\frac{1}{\sqrt{2\pi}}\int_{-\infty}^{\infty}e^{-\frac{x^2}{2}(1-t^2)}\left(\frac{1}{\sqrt{2\pi}}\int_{-\infty}^{\infty}e^{-\frac{(y-tx)^2}{2}}dy\right)dx\\
    =&\frac{1}{\sqrt{2\pi}}\int_{-\infty}^{\infty}e^{-\frac{x^2}{2}(1-t^2)}dx\\
    =&\frac{1}{\sqrt{1-t^2}}.
\end{align*}
\end{proof}
We introduce the following lemma that helps in our calculation later.
\begin{lemma}\label{le.temp_0125_1}
For any $x>0$,
\begin{align*}
    \argmax_{t\in(0, 1)}\ \left(tx+\frac{n}{2}\ln(1-t^2)\right)=\frac{-n+\sqrt{n^2+4x^2}}{2x}.
\end{align*}
\end{lemma}
\begin{proof}
Let
\begin{align*}
    f(t) = tx+\frac{n}{2}\ln(1-t^2),\quad t\in(0, 1).
\end{align*}
Then, we have
\begin{align*}
    \frac{df(t)}{dt}=x-\frac{nt}{1-t^2}.
\end{align*}
Letting $df(t)/dt=0$, we have exactly one solution in $(0,1)$ given by
\begin{align*}
    t=\frac{-n+\sqrt{n^2+4x^2}}{2x}.
\end{align*}
Notice that $df(t)/dt$ is monotone decreasing with respect to $t$ and thus $f(t)$ is concave on $(0,1)$. The result of this lemma thus follows.
\end{proof}
We then use the Chernoff bound to estimate $\mathbf{H}_i^T\mathbf{H}_j$ in the following lemma.
\begin{lemma}\label{le.ub_HH}
\begin{align*}
    &\prob{\mathbf{H}_i^T\mathbf{H}_j>\frac{na}{2}}\\
    &\leq \exp\left({-\frac{n}{2}\left(at+\ln\frac{2t}{a}\right)}\right),
\end{align*}
where
\begin{align*}
    t=\frac{-1+\sqrt{1+a^2}}{a}.
\end{align*}
\end{lemma}
\begin{proof}
Notice that
\begin{align*}
    \mathbf{H}_i^T\mathbf{H}_j=\sum_{k=1}^n\mathbf{H}_{ik}\mathbf{H}_{jk}=\sum_{k=1}^n Z_k,
\end{align*}
where $Z_k\defeq \mathbf{H}_{ik}\mathbf{H}_{jk}$. 
Using the Chernoff bound, we have
\begin{align*}
    \prob{\mathbf{H}_i^T\mathbf{H}_j>x}\leq& \min_{t>0}e^{-tx}\prod_{k=1}^n\mathds{E}[e^{tZ_k}]
\end{align*}
Since each $Z_k$ is the product of two independent standard Gaussian variable, using Lemma \ref{le.temp_0125_1}, we have, for any $x>0$,
\begin{align*}
    \prob{\mathbf{H}_i^T\mathbf{H}_j>x}\leq&\min_{t>0}e^{-tx}(1-t^2)^{-\frac{n}{2}}\\
    =&\min_{t\in (0,1)}e^{-tx}(1-t^2)^{-\frac{n}{2}}\\
    =&\min_{t\in (0,1)}e^{-tx -\frac{n}{2}\ln(1-t^2)}\\
    =&\exp\left({-tx -\frac{n}{2}\ln(1-t^2)}\right)\bigg|_{t=\frac{-n+\sqrt{n^2+4x^2}}{2x}}\text{ (by Lemma \ref{le.temp_0125_1})}\\
    =&\exp\left({-tx -\frac{n}{2}\ln(nt/x)}\right)\bigg|_{t=\frac{-n+\sqrt{n^2+4x^2}}{2x}},
\end{align*}
where the last equality is because $t=({-n+\sqrt{n^2+4x^2}})/{2x}$ is one solution of the quadratic equation in $t$ that $xt^2+nt-x=0$ (which implies $1-t^2=nt/x$).

Letting $x=\frac{na}{2}$, we get $t=(-1+\sqrt{1+a^2})/a$, and 
\begin{align*}
    \exp\left({-tx -\frac{n}{2}\ln(nt/x)}\right)=\exp\left(-\frac{nat}{2}-\frac{n}{2}\ln \frac{2t}{a} \right)=\exp\left(-\frac{n}{2}\left(at+\ln \frac{2t}{a}\right)\right).
\end{align*}
The result of this lemma thus follows.
\end{proof}

\subsection*{Step 3: Estimating the distribution of $M$.}
Since $M$ is defined as the maximum of all $|G_{ij}|$ for $i\neq j$, we use the union bound to estimate the distribution of $M$ in the following proposition.
\begin{proposition}\label{prop.draw_M_upper}
\begin{align*}
    &\prob{M\leq 2\sqrt{6}\sqrt{\frac{\ln p}{n}\left(\frac{6\ln p}{n}+1\right)}}
    \geq 1-2e^{-\ln p}-2e^{-n/16+2\ln p}.
\end{align*}
\end{proposition}
To prove Proposition \ref{prop.draw_M_upper}, we introduce a technique lemma first.
\begin{lemma}\label{le.exp_bound}
For any $x>0$, we must have
\begin{align*}
    \ln x\geq 1-\frac{1}{x}.
\end{align*}
\end{lemma}
\begin{proof}
We define a function
\begin{align*}
    f(x)\defeq \ln x - (1-\frac{1}{x}),\quad x>0.
\end{align*}
It suffices to show that $\min f(x)=0$. We have
\begin{align*}
    \frac{df(x)}{dx}=\frac{1}{x}-\frac{1}{x^2}=\frac{x-1}{x^2}.
\end{align*}
Thus, $f(x)$ is monotone decreasing in $(0, 1)$ and monotone increasing in $(1, \infty)$. Thus, $\min f(x)=f(1)=0$. The conclusion of this lemma thus follows.
\end{proof}
We are now ready to prove Proposition \ref{prop.draw_M_upper}.

\noindent\begin{myproof}{Proposition \ref{prop.draw_M_upper}}
Applying Lemma \ref{le.ub_HH} to Eq. (\ref{eq.mid_result_M}), we have
\begin{align}\label{eq.temp_0125_5}
    &\prob{|G_{ij}|>a} 
    \leq 2\exp\left({-\frac{n}{2}\left(at+\ln \frac{2t}{a}\right)}\right)+2e^{-n/16},
\end{align}
where
\begin{align}\label{eq.temp_0125_6}
    t=\frac{-1+\sqrt{1+a^2}}{a}.
\end{align}
Since $M=\max_{i\neq j}|G_{ij}|$, we have
\begin{align}
    &\prob{M\leq a}\nonumber\\
    =&1-\Pr\left(\bigcup_{i\neq j}\left\{|G_{ij}|> a\right\}\right)\nonumber\\
    \geq& 1-\sum_{i\neq j}\prob{|G_{ij}|>a}\text{ (by the union bound)}\nonumber\\
    =&1-p(p-1)\prob{|G_{ij}|>a}\text{ (since all $G_{ij}$ has the same distribution)}\nonumber\\
    \geq &1-e^{2\ln p}\prob{|G_{ij}|>a}\nonumber\\
    \geq &1-2e^{-n/16+2\ln p}\nonumber\\
    &-2\exp\left({-\frac{n}{2}\left(at+\ln \frac{2t}{a}-\frac{4\ln p}{n}\right)}\right)\text{ (by Eq. \eqref{eq.temp_0125_5})}.\label{eq.temp_0125_10}
\end{align}
Let
\begin{align}\label{eq.temp_0125_7}
    a =2\sqrt{6}\sqrt{\frac{\ln p}{n}\left(\frac{6\ln p}{n}+1\right)}.
\end{align}
Substituting Eq. \eqref{eq.temp_0125_7} into Eq. \eqref{eq.temp_0125_6}, we have
\begin{align}
    at&=-1+\sqrt{1+a^2}\nonumber\\
    &=-1+\sqrt{1+\frac{24\ln p}{n}+\left(\frac{12\ln p}{n}\right)^2}\nonumber\\
    &=-1+\sqrt{\left(\frac{12\ln p}{n}+1\right)^2}\nonumber\\
    &=\frac{12\ln p}{n}.\label{eq.temp_0125_8}
\end{align}
Thus, we have
\begin{align}
    \ln\frac{2t}{a}=\ln\frac{2at}{a^2}=\ln \frac{2\cdot\frac{12\ln p}{n}}{24\cdot\frac{\ln p}{n}\left(\frac{6\ln p}{n}+1\right)}&=\ln\frac{1}{\frac{6\ln p}{n}+1}
    \nonumber\\
    &\geq 1-\left(\frac{6\ln p}{n}+1\right)\text{ (by Lemma \ref{le.exp_bound})}\nonumber\\
    &=-\frac{6\ln p}{n}.\label{eq.temp_0125_9}
\end{align}
By Eq. \eqref{eq.temp_0125_8} and Eq. \eqref{eq.temp_0125_9}, we have
\begin{align*}
    -\frac{n}{2}\left(at+\ln\frac{2t}{a}-\frac{4\ln p}{n}\right)\leq -\frac{n}{2}\left(\frac{12\ln p}{n}-\frac{6\ln p}{n}-\frac{4\ln p}{n}\right)= -\ln p.
\end{align*}
Substituting into Eq. \eqref{eq.temp_0125_10}, the result of this proposition follows.
\end{myproof}

\subsection*{Step 4: Simplifying the expression in Proposition \ref{prop.draw_M_upper}.}
By the assumption of Proposition \ref{prop.M} that $p\leq \exp(n/36)$, we have
\begin{align*}
    \frac{6\ln p}{n}+1\leq \frac{7}{6}.
\end{align*}
Thus, we have
\begin{align}\label{eq.temp_0126_1}
    2\sqrt{6}\sqrt{\frac{\ln p}{n}\left(\frac{6\ln p}{n}+1\right)}\leq 2\sqrt{7}\sqrt{\frac{\ln p}{n}}.
\end{align}
We also have
\begin{align}\label{eq.temp_0126_2}
    \frac{-n}{16}+2\ln p\leq \frac{-n}{16}+2\cdot\frac{n}{36}=-\frac{n}{144}.
\end{align}
Applying Eq. \eqref{eq.temp_0126_1} and Eq. \eqref{eq.temp_0126_2} to Proposition \ref{prop.draw_M_upper}, we then get Proposition \ref{prop.M}.

\section{Lower bounds}
\label{app:lower_bounds}

In this section, we first establish a lower bound on $\|w^I\|_1$. This lower
bound now only shows that our upper bound in Prop.~\ref{prop.new_WI} is
tight (up to a constant factor), but can also be used to derive a lower
bound on $\|\wBP\|_1$. We will then use this lower bound on $\|\wBP\|_1$
to prove Prop.~\ref{prop:lower:BP} (i.e., the lower bound on
$\|\wBP\|_2$). As we
discussed in the main body of the paper, although our bounds on
$\|\wBP\|_2$ are not tight, the bounds on $\|\wBP\|_1$ are in fact tight
(up to a constant factor), which will be shown below.


\subsection{Lower bound on \texorpdfstring{$\|w^I\|_1$}{||wI||1}}

A trivial lower bound on $\|w^I\|_1$ is  $\|w^I\|_1\geq \|\etr\|_2$. To
see this, letting $w^I_{(i)}$ denote the $i$-th element of $w^I$, we have
\begin{align*}
    &\|\etr\|_2=\|\Xtr w^I\|_2=\left\|\sum_{i=1}^p w^I_{(i)}\X_i\right\|_2\\
    &\leq \sum_{i=1}^p |w^I_{(i)}|\cdot\|\X_i\|_2=\|w^I\|_1 \text{ (notice $\|\X_i\|_2=1$)}.
\end{align*}
Even by this trivial lower bound, we immediately know that our upper
bound on $\|w^I\|_1$ in Proposition \ref{prop.new_WI} is accurate when
$p\rightarrow \infty$. Still, we can do better than this trivial lower
bound, as shown in Proposition \ref{prop.lower_wI} below.

Towards this end, following the construction of Problem \eqref{prob.temp1}, it is not hard to show that $\mathbf{B}_{(1)}$, i.e., the vector that has the largest inner-product with $(-\etr)$, defines a lower bound for $\|w^I\|_1$.
\begin{lemma}\label{le.true_lower_bound_WI}
\begin{align*}
    \|w^I\|_1\geq \frac{\|\etr\|_2^2}{\mathbf{B}_{(1)}^T(-\etr)}
\end{align*}
\end{lemma}
\begin{proof}
Let
\begin{align*}
    \lambda_*=\frac{(-\etr)}{\mathbf{B}_{(1)}^T(-\etr)}.
\end{align*}
By the definition of $\mathbf{B}_{(1)}$, for any $i\in\{1,2,\cdots, p-s\}$, we have
\begin{align*}
    \left|\lambda_*^T\mathbf{A}_i\right|=\frac{\left|\mathbf{A}_i^T\etr\right|}{|\mathbf{B}_{(1)}^T\etr|}\leq 1.
\end{align*}
In other words, $\lambda_*$ satisfies all constraints of the problem (\ref{prob.dual}), which implies that the optimal objective value of \eqref{prob.dual} is at least
\begin{align*}
    \lambda_*^T(-\etr)=\frac{\|\etr\|_2^2}{\mathbf{B}_{(1)}^T(-\etr)}.
\end{align*}
The result of this lemma thus follows.
\end{proof}

By bounding $\mathbf{B}_{(1)}^T(-\etr)$, we can show the following result.
\begin{proposition}\label{prop.lower_wI}
When $p\leq e^{(n-1)/16}/n$ and $n\geq 17$, then
\begin{align*}
    \frac{\|w^I\|_1}{\|\etr\|_2}\geq \sqrt{1+\frac{n}{9\ln p}}
\end{align*}
with probability at least $1-3/n$.
\end{proposition}
The proof is available in Appendix \ref{app.proof_prop_14}. Comparing
Proposition~\ref{prop.new_WI} with Proposition \ref{prop.lower_wI}, we
can see that, with high probability, the upper and lower bounds of
$\|w\|_1$ differ by at most a constant factor.

\subsection{Lower bounds on \texorpdfstring{$\|\wBP\|_1$}{||wBP||1} and \texorpdfstring{$\|\wBP\|_2$}{||wBP||2}}

Using Prop.~\ref{prop.lower_wI}, we can show the following lower bound
on $\|\wBP\|_1$. 

\begin{proposition}[lower bound on $\|\wBP\|_1$]\label{th.lower_bound_l1}
When $p\leq e^{(n-1)/16}/n$ and $n\geq 17$, then
\begin{align*}
    \|\wBP\|_1\geq \frac{1}{3}\sqrt{\frac{n}{\ln p}}\|\etr\|_2
\end{align*}
with probability at least $1-3/n$.
\end{proposition}
\begin{proof}
We define $w^J$ as the solution to the following optimization problem:
\begin{align*}
    \min_w \|w\|_1,\ \st \Xtr w=\etr.
\end{align*}
By definition, $\Xtr \wBP=\etr$. Thus, we have $\|\wBP\|_1\geq \|w^J\|_1$. To get a
	lower bound on $\|w^J\|_1$, we can directly use the result in
	Proposition~\ref{prop.lower_wI} because the definitions of $w^I$
	and $w^J$ are essentially the same\footnote{Notice that the
	proof of Proposition~\ref{prop.lower_wI} does not require
	$s>0$. Therefore, we can just let $s=0$ so that $w^I$ there
	becomes $w^J$.}. We
	then have, 
\begin{align*}
    \|w^J\|_1\geq \sqrt{1+\frac{n}{9\ln p}}\|\etr\|_2\geq \frac{1}{3}\sqrt{\frac{n}{\ln p}}\|\etr\|_2
\end{align*}
with probability at least $1-3/n$. The result of this proposition thus follows.
\end{proof}

Next, we proceed to prove Proposition~\ref{prop:lower:BP}, i.e., the
lower bound on $\|\wBP\|_2$. 
Because $\|\wBP\|_0=\|\betaBP-\beta\|_0\leq \|\betaBP\|_0+\|\beta\|_0\leq n+s$, we then have the following lower bound on
$\|\wBP\|_2$ assuming $n \ge s$, 
\begin{align}\label{eq.lower_wB2}
    \|\wBP\|_2\geq \frac{\|\wBP\|_1}{\sqrt{n+s}}\geq \frac{\|\wBP\|_1}{\sqrt{2n}}.
\end{align}
Combining with Prop.~\ref{th.lower_bound_l1}, we have proved
Prop.~\ref{prop:lower:BP}.

\subsection{Tightness of the bounds on \texorpdfstring{$\|\wBP\|_1$}{||wBP||1}}

As we discussed in the main body of the paper, our upper and lower bounds on
$\|\wBP\|_2$ still have a significant gap. Interesting, our bounds on
$\|\wBP\|_1$ are tight up to a constant factor, which may be of
independent interest. To show this, we first derive the following upper
bound on $\|\wBP\|_1$. 

{
\begin{proposition}[upper bound on $\|\wBP\|_1$]
When $s\leq \sqrt{\frac{n}{7168\ln (16n)}}$, if $p\in \left[(16n)^4,\ \exp\left(\frac{n}{1792s^2}\right)\right]$, then
\begin{align*}
    \|\wBP\|_1\leq \left(4\sqrt{2}+\sqrt{\frac{1}{2\sqrt{7}}}\right)\sqrt{\frac{n}{\ln p}}\|\etr\|_2,
\end{align*}
with probability at least $1-6/p$.
\end{proposition}
\begin{proof}
Following the proof of Theorem~\ref{th.main} in Appendix~\ref{app:proof:main}, we can still get that Eq.~(\ref{eq.estimate_M_wI_K}), i.e.,
\begin{align}
    M\leq 2\sqrt{7}\sqrt{\frac{\ln p}{n}},\ 
    \|w^I\|_1\leq \sqrt{\frac{2n}{\ln p}}\|\etr\|_2,\text{ and }K\geq 4,\label{eq.temp_060601}
\end{align}
hold with probability at least $1-6/p$.
Applying Eq.~(\ref{eq.temp_060601}) and Proposition~\ref{prop.bound_WB1_KM}, we have, with probability at least $1-6/p$,
\begin{align*}
    \|\wBP\|_1\leq & 4\|w^I\|_1+\sqrt{\frac{1}{2\sqrt{7}}}\left(\frac{n}{\ln p}\right)^{1/4}\|\etr\|_2\\
    \leq& 4\sqrt{\frac{2n}{\ln p}}\|\etr\|_2+\sqrt{\frac{1}{2\sqrt{7}}}\left(\frac{n}{\ln p}\right)^{1/4}\|\etr\|_2\\
    \leq & \left(4\sqrt{2}+\sqrt{\frac{1}{2\sqrt{7}}}\right)\sqrt{\frac{n}{\ln p}}\|\etr\|_2,
\end{align*}
where the last inequality is because $\frac{n}{\ln p}>1$, and therefore $(\frac{n}{\ln p})^{1/4}\leq (\frac{n}{\ln p})^{1/2}$.
\end{proof}

Comparing with Prop.~\ref{th.lower_bound_l1}, we can see that our upper
and lower bounds on $\|\wBP\|_1$ differ by at most a constant factor. 

}

\subsection{Proof of Proposition \ref{prop.lower_wI}}\label{app.proof_prop_14}
To prove Proposition \ref{prop.lower_wI}, we will prove a slightly stronger result in Proposition \ref{prop.proposition_stronger} given below.
\begin{proposition}\label{prop.proposition_stronger}
When $(p-s)\leq e^{(n-1)/16}/n$ and $n\geq 17$, the following holds.
\begin{align}\label{eq.lower_bound_WI1}
    \frac{\|w^I\|_1}{\|\etr\|_2}\geq \sqrt{1+\frac{n-1}{4\ln n+4\ln (p-s)}},
\end{align}
with probability at least $1-3/n$.
\end{proposition}
To prove Proposition \ref{prop.proposition_stronger}, we introduce a technical lemma first.
\begin{lemma}\label{le.bound_log}
For any $x\in[0, 1)$, we have
\begin{align}\label{eq.upper_bound_WI1}
    \ln(1-x)\geq \frac{-x}{\sqrt{1-x}}.
\end{align}
\end{lemma}
\begin{proof}
Let
\begin{align*}
    f(x)=\ln(1-x)+\frac{x}{\sqrt{1-x}}.
\end{align*}
Note that $f(0)=0$. Thus, it suffices to show that $df(x)/dx\geq 0$ when $x\in[0, 1)$. Indeed, we have
\begin{align*}
    \frac{d f(x)}{d x}=&\frac{-1}{1-x}+\frac{\sqrt{1-x}-x\frac{-1}{2\sqrt{1-x}}}{1-x}\\
    =&\frac{-\sqrt{1-x}+1-x+x/2}{(1-x)^{3/2}}\\
    =&\frac{2-x-2\sqrt{1-x}}{2(1-x)^{3/2}}\\
    =&\frac{(1-\sqrt{1-x})^2}{2(1-x)^{3/2}}\\
    \geq& 0.
\end{align*}
The result of this lemma thus follows.
\end{proof}
We are now ready to prove Proposition \ref{prop.proposition_stronger}.

\noindent\begin{myproof}{Proposition \ref{prop.proposition_stronger}}
Because of Lemma \ref{le.true_lower_bound_WI}, we only need to show that
\begin{align*}
    \frac{\|\etr\|_2}{\mathbf{B}_{(1)}^T(-\etr)}\geq\sqrt{1+\frac{n-1}{4\ln n+4\ln (p-s)}},
\end{align*}
with probability at least $1-3/n$.
Similar to what we do in Appendix \ref{app.proof_of_lemma_12}, without loss of generality, we let $\etr=[-\|\etr\|_2\ \ 0\ \ \cdots\ \ 0]^T$. Thus,
\begin{align*}
    \frac{\|\etr\|_2}{\mathbf{B}_{(1)}^T(-\etr)}=\frac{1}{\max_i|\mathbf{A}_{i1}|}.
\end{align*}
We uses the following two steps in order to get an upper bound of $1/{\max_i|\mathbf{A}_{i1}|}$. Step 1: estimate the distribution of $1/{|\mathbf{A}_{i1}|}$ for any $i\in \{1,\cdots,p-s\}$. Step 2: utilizing the fact that all $\mathbf{A}_{i1}$'s are independent, we estimate $1/{\max_i|\mathbf{A}_{i1}|}$ base on the result in Step 1.

The Step 1 proceeds as following.
For any $i\in \{1,\cdots,p-s\}$ and any $k\geq 0$, we have
\begin{align*}
    &\prob{\frac{1}{|\mathbf{A}_{i1}|}\geq k}\\
    =&\prob{(\mathbf{A}'_{i1})^2\leq \frac{\sum_{j=2}^n(\mathbf{A}'_{ij})^2}{k^2-1}}\text{ (by Eq. \eqref{eq.temp_0124_2})}.
\end{align*}
Therefore, for any $m>0$, we have
\begin{align}
    &\prob{\frac{1}{|\mathbf{A}_{i1}|}\geq k}\nonumber\\
    \geq & \prob{(\mathbf{A}'_{ij})^2\leq \frac{n-1-2\sqrt{(n-1)m}}{k^2-1}}\nonumber\\
    &\cdot\prob{\sum_{j=2}^n(\mathbf{A}'_{ij})^2>n-1-2\sqrt{(n-1)m}}\text{ (because all $\mathbf{A}'_{ij}$'s are independent)}\nonumber\\
    \geq & \left(1-2\Phi^c\left(\sqrt{\frac{n-1-2\sqrt{(n-1)m}}{k^2-1}}\right)\right)\nonumber\\
    &\cdot\left(1-e^{-m}\right)\text{ (by Lemma \ref{le.chi_bound})}.\label{eq.temp_0124_4}
\end{align}
Let $m=(n-1)/16$ and define
\begin{align}\label{eq.temp_def_t}
    t\defeq\sqrt{\frac{(n-1)/2}{k^2-1}}.
\end{align}
We have
\begin{align}\label{eq.temp_0124_3}
    \sqrt{\frac{n-1-2\sqrt{(n-1)m}}{k^2-1}}=t.
\end{align}
Substituting Eq. \eqref{eq.temp_0124_3} and $m=(n-1)/16$ to Eq. \eqref{eq.temp_0124_4}, we have
\begin{align*}
    &\prob{\frac{1}{|\mathbf{A}_{i1}|}\geq k}\geq \left(1-e^{-(n-1)/16}\right)(1-2\Phi^c(t))\\
    &\geq \left(1-e^{-(n-1)/16}\right)\left(1-\frac{2\sqrt{2/\pi}e^{-t^2/2}}{t+\sqrt{t^2+\frac{8}{\pi}}}\right) \text{(by Lemma \ref{le.estimate_phi_c})}\\
    &\geq \left(1-e^{-(n-1)/16}\right)\left(1-e^{-t^2/2}\right)\text{ (since $t\geq 0\implies t+\sqrt{t^2+8/\pi}\geq 2\sqrt{2/\pi}$)}.
\end{align*}
Now, let $k$ take the value of the RHS of Eq. \eqref{eq.lower_bound_WI1}, i.e.,
\begin{align*}
    k=\sqrt{1+\frac{n-1}{4\ln n+4\ln (p-s)}}.
\end{align*}
By Eq. \eqref{eq.temp_def_t}, we have
\begin{align*}
    t^2=&\frac{(n-1)/2}{k^2-1}\\
    =&\frac{(n-1)/2}{\left(\sqrt{1+\frac{n-1}{4\ln n+4\ln (p-s)}}\right)^2-1}\text{ (substituting the value of $k$)}\\
    =&2\ln n+2\ln (p-s),
\end{align*}
which implies that
\begin{align*}
    e^{-t^2/2}=\frac{1}{n(p-s)}.
\end{align*}
Thus, we have
\begin{align}\label{eq.temp_0124_1}
    &\prob{\frac{1}{|\mathbf{A}_{i1}|}\geq k}\geq \left(1-e^{-(n-1)/16}\right)\left(1-\frac{1}{n(p-s)}\right).
\end{align}
Next, in Step 2, we use Eq. \eqref{eq.temp_0124_1} to estimate $1/\max_i|\mathbf{A}_{i1}|$. Since all $\mathbf{A}_{i1}$'s are independent, we have
\begin{align*}
    &\prob{\frac{1}{\max_i|\mathbf{A}_{i1}|}\geq k}\\
    =&\prod_{i=1}^{p-s} \prob{\frac{1}{|\mathbf{A}_{i1}|}\geq k}\text{ (since all $\mathbf{A}_{i1}$ are independent)}\\
    \geq &\left(\left(1-e^{-(n-1)/16}\right)\left(1-\frac{1}{n(p-s)}\right)\right)^{p-s}\text{ (by Eq. \eqref{eq.temp_0124_1})}\\
    =&\exp\left({(p-s)\ln(1-e^{-(n-1)/16})}\right)\\
    &\cdot\exp\left({(p-s)\ln(1-\frac{1}{n(p-s)})}\right)\\
    \geq &\exp\left({-\frac{(p-s)e^{-(n-1)/16}}{\sqrt{1-e^{-(n-1)/16}}}}\right)\exp\left({(p-s)\frac{-\frac{1}{n(p-s)}}{\sqrt{1-\frac{1}{n(p-s)}}}}\right)\\
    &\text{ (by Lemma \ref{le.bound_log})}\\
    =&\exp\left({-\frac{(p-s)e^{-(n-1)/16}}{\sqrt{1-e^{-(n-1)/16}}}}\right)\exp\left({\frac{-1}{n\sqrt{1-\frac{1}{n(p-s)}}}}\right)\\
    \geq &\left(1-\frac{(p-s)e^{-(n-1)/16}}{\sqrt{1-e^{-(n-1)/16}}}\right)\left(1-\frac{1}{n\sqrt{1-\frac{1}{n(p-s)}}}\right)\\
    &\text{ (because $e^x\geq 1+x$)}\\
    \geq &\left(1-\frac{1}{n\sqrt{1-e^{-(n-1)/16}}}\right)\left(1-\frac{1}{n\sqrt{1-1/17}}\right)\\
    &\text{ (based on the assumption of the proposition, i.e., $p-s\leq e^{(n-1)/16}/n$ and $n(p-s)\geq n \geq 17$)}\\
    \geq &\left(1-\frac{1}{n\sqrt{1-1/e}}\right)\left(1-\frac{1}{n\sqrt{1-1/17}}\right)\text{ (because $n\geq 17$)}\\
    =& 1 - \frac{1}{n\sqrt{1-1/e}} - \frac{1}{n\sqrt{1-1/17}}+\frac{1}{n\sqrt{1-1/e}}\frac{1}{n\sqrt{1-1/17}}\\
    \geq & 1-\frac{2}{\sqrt{1-1/e}}\cdot\frac{1}{n}\text{ (because $17>e$)}\\
    \geq & 1-3/n\text{ (because $e\geq 9/5$)}.
\end{align*}
The result of this proposition thus follows.
\end{myproof}

Finally, we use the following lemma to simplify the expression in Proposition \ref{prop.proposition_stronger}. The result of Proposition \ref{prop.lower_wI} thus follows.
\begin{lemma}\label{le.9lnp}
If $n\geq 17$, then
\begin{align*}
    \sqrt{1+\frac{n-1}{4\ln p + 4\ln (p-s)}}\geq \sqrt{1+\frac{n}{9\ln p}}.
\end{align*}
\end{lemma}
\begin{proof}
Because $n\geq 17$, we have
\begin{align*}
    \frac{n-1}{n}=1-\frac{1}{n}\geq 1-\frac{1}{17}\geq \frac{8}{9}.
\end{align*}
Therefore, we have
\begin{align*}
    &\frac{n-1}{n}\geq \frac{4}{9}+\frac{4}{9}\\
    \implies & \frac{n-1}{n}\geq \frac{4\ln p}{9\ln p}+\frac{4\ln (p-s)}{9\ln p}\\
    \implies & \frac{n-1}{n}\geq \frac{4\ln p+ 4\ln (p-s)}{9\ln p}\\
    \implies & \frac{n-1}{4\ln p+ 4\ln (p-s)}\geq \frac{n}{9\ln p}\\
    \implies & \sqrt{1+\frac{n-1}{4\ln p + 4\ln (p-s)}}\geq \sqrt{1+\frac{n}{9\ln p}}.
\end{align*}
\end{proof}

\end{document}